%% file: main.tex
\documentclass[conference]{IEEEtran}
\usepackage{tikz}
\usepackage{amsmath,amsfonts}
\usepackage{amsthm}

\usepackage{amssymb} 
\usepackage{algorithm}
\usepackage{algorithmic}
\usepackage{bm}
\usepackage{graphicx}
\usepackage{textcomp}
\usepackage{xcolor}
\usepackage{subfigure}
\usepackage{booktabs}
\usepackage{wrapfig}
\usepackage{multirow}
\usepackage{listings}
\usepackage{url}

\newtheorem{theorem}{\bf Theorem}
\newtheorem{lemma}{\bf Lemma}

\newtheorem{definition}{\bf Definition}

\begin{document}
%-------------------------------------------------------------------------------

%don't want date printed
\date{}

\title{Weights Shuffling for Improving DPSGD in Transformer-based Models}

%for single author (just remove % characters)
\author{
{\rm Jungang Yang}\\
Shanghai Jiao Tong University\\
\textit{yangjungang@shu.edu.cn}
\and
{\rm Zhe Ji}\\
Shanghai Jiao Tong University\\
\textit{ji\_zhe@sjtu.edu.cn}
% copy the following lines to add more authors
\and
{\rm Liyao Xiang}\\
Shanghai Jiao Tong University\\
\textit{xiangliyao08@sjtu.edu.cn}
} % end author

\IEEEoverridecommandlockouts
\makeatletter\def\@IEEEpubidpullup{6.5\baselineskip}\makeatother

\maketitle
\begin{abstract}
	Differential Privacy (DP) mechanisms, especially in high-dimensional settings, often face the challenge of maintaining privacy without compromising the data utility. This work introduces an innovative shuffling mechanism in Differentially-Private Stochastic Gradient Descent (DPSGD) to enhance the utility of large models at the same privacy guarantee of the unshuffled case. Specifically, we reveal that random shuffling brings additional randomness to the trajectory of gradient descent while not impacting the model accuracy by the permutation invariance property --- the model can be equivalently computed in both forward and backward propagations under permutation. We show that permutation indeed improves the privacy guarantee of DPSGD in theory, but tracking the exact privacy loss on shuffled model is particularly challenging. Hence we exploit the approximation on sum of lognormal distributions to derive the condition for the shuffled DPSGD to meet the DP guarantee. Auditing results show that our condition offers a DP guarantee quite close to the audited privacy level, demonstrating our approach an effective estimation in practice. Experimental results have verified our theoretical derivation and illustrate that our mechanism improves the accuracy of DPSGD over the state-of-the-art baselines on a variety of models and tasks.
\end{abstract}

% \keywords{Differential Privacy, Differentially Private Stochastic Gradient Descent, Shuffling, Large Models}

\input{arxiv/intro.tex}

\input{arxiv/related.tex}

\input{arxiv/prelim.tex}

\input{arxiv/method.tex}

\input{arxiv/algorithm.tex}

\input{arxiv/experiment.tex}

\input{arxiv/conclusion.tex}

\bibliographystyle{arxiv/IEEEtranS}
\bibliography{arxiv/main.bib}

\cleardoublepage
\input{arxiv/appendix.tex}

%%%%%%%%%%%%%%%%%%%%%%%%%%%%%%%%%%%%%%%%%%%%%%%%%%%%%%%%%%%%%%%%%%%%%%%%%%%%%%%%
\end{document}

%% file: arxiv/intro.tex
\section{Introduction}

% background of DP
Recent work has revealed that trained neural networks pose great threats in leaking sensitive training data. Differential Privacy (DP) \cite{dwork2006calibrating} serves both as a measure to quantify the upper bound of such information leak, and as mechanisms to ensure any individual sample's impact on the model is negligible. The core principle of DP is to introduce random perturbation to the queried data, effectively preventing the disclosure of an individual sample. Differentially-private models have shown some defensive capability against privacy attacks including membership inference attacks \cite{RahmanRLM18,Sablayrolles2019WhiteboxVB,yu2021large}, gradient matching attack \cite{Zhu2019DeepLF}, input reconstruction attack \cite{Carlini2019secret}, etc.

% background of DPSGD
Among the DP mechanisms, Differentially-Private Stochastic Gradient Descent (DPSGD) \cite{abadi2016deep} emerges as an important one in privacy-preserving machine learning. By clipping the gradients of each example and inserting randomized noise to the clipped gradients in each training iteration, DPSGD ensures that the published model satisfies DP over the private training dataset.  However, the conventional methods suffer significant accuracy losses due to the overwhelming amount of noise inserted, especially on deep and large models. Traditional DPSGD techniques, while effective in lower dimensions, struggle to maintain data utility as the model size enlarges. 

% highlight the challenge in DPSGD, Low-rank and tuning parameter.
Many works have been devoted into tackling the high-dimensional issue in DPSGD. Some have observed that the gradients of deep neural networks mostly live on a very low dimensional manifold \cite{vogels2019powersgd,goo2020lowrank,li2022low}, and thus can be compressed without hurting the model performance too much. Based on the observation, a series of works \cite{yu2021large,Huanyu2021Wide,yu2021not,zhou2021bypassing,yu2022differentially} propose to replace the full gradients with their low-dimensional projection to reduce the overall noise added. Despite their impressive performance, most works avoid paying privacy price for the acquisition of low-dimensional representations, either by modifying the original model structure, or relying on auxiliary public data to estimate the low-dimensional projecting subspace. These works hardly keep the original training on large models as it is. Other works improve the accuracy of DPSGD by hyperparameters tuning \cite{de2022unlocking,li2022large}. However, hyperparameters tuning consumes a great amount of computational resource in large-scale training. Moreover, such solutions merely provide training tricks which may not be transferable to other models. 

%The motivation of shuffled-DPSGD
% 为了解决高维度带来的差分隐私的挑战，在不牺牲高维度的模型的优势前提下，我们这里采用shuffle对DP的隐私进行提升。首先高维度差分隐私机制中，高维度的数据需要更大的扰动对隐私进行保护。然后过大的噪声扰动将验证影响模型性能，因此我们引入shuffle扰动，降低了差分隐私对噪声扰动的需要，从而降低了噪声的方差，提升了模型的性能。
We take a drastically different approach to mitigate the accuracy loss of DPSGD on large models --- rather than blaming on the high dimensionality, we take advantage of it. As we observe, for large models, the possible number of weights permutations could be huge, \textit{e.g.}, the permutation space for a billion-level model such as GPT-2 is as large as $10^9 !$. Hence if we could randomly permute the model weights without affecting the model utility at each training step, we would obscure the training trajectory thus gaining privacy with no utility cost, alleviating the amount of additive perturbation noise. 

One should note that there is a significant distinction between our methodology and the existing shuffle model of privacy \cite{bittau2017prochlo,Erlingsson2019Amplification,cheu2019distributed,balle2019privacy,Kairouz2021Practical,feldman2021hiding}. The latter is a distributed model in which a shuffler permutes the locally-differentially-private reports from different \textit{users} for obfuscation. The shuffled reports are then released to a server for analysis. It has been shown that the random shuffling of inputs to locally private protocols amplifies the privacy guarantee \cite{feldman2022hiding}. In contrast, our approach works in the central DP scenario focusing on the shuffling of \textit{model weights} instead of \textit{user data} to hide the original weights among a set of weights which are equivalent in computing the output. The random shuffling enhances privacy as it obfuscates the training trajectories on the dataset hereby preserving training data privacy at a low accuracy cost. Moreover, we confront a unique challenge: the existing shuffle model of privacy naturally preserves accuracy as the user shuffling would not affect the computation, while a naive random shuffling of model weights would totally destroy the model utility.

%Given the intuition, we still face several challenges: \textit{1) how can we guarantee shuffling not impacting model utility? 2)  can we quantify the randomness brought by shuffling through DP? 3) can shuffling be integrated seamlessly into DPSGD?} We will answer these questions one by one in the following.

%介绍shuffle机制对模型的影响
% 虽然引入shuffle扰动降低了噪声的影响，但依然面临新的挑战：引入的shuffle能否降低模型精度上的扰动？如果进行完全随机的shuffle，模型的输出会发生巨大的变化。因此这里我们提出了Permutation Invariance的概念，我们需要保证shuffle扰动后的模型输出和原本的模型输出一致。也就是说于任意的输入，shuffle扰动后的模型输出和原本的模型输出一致。这里我们主要通过对模型的线性层和注意力机制层进行定制化的shuffle机制设计，从而保证了模型的输出不变性。这里我们需要强调的是，我们的shuffle机制并不是针对所有网络都成立。这里我们主要在线性层和注意力机制层可以实现扰动不变，但不包括卷积层。

Fortunately, as we analyzed the mainstream neural architecture, we found that multilayer perceptron (MLP) and Transformer encoder blocks, both taking a large proportion of the model weights, satisfy permutation invariance property. Specifically, if their weight matrices are permuted, the permutation can be negated by an inverted permutation in the following layer without disrupting the forward and backward propagation. Hence weight permutation incurs zero accuracy loss on the inference or training of the model. Apart from that, MLP and Transformer encoder blocks widely exist in the state-of-the-art neural networks, especially large models with Transformer backbones, and these parts of a network are mostly dense, creating a huge space for permutation without any accuracy loss.

%我们主要通过对Gaussian Mechanism的结果进行shuffle增强原本Gaussian机制的隐私。原本差分隐私是通过衡量两个均值不同的高斯分布的worst case距离。但现在经过shuffle后，原本的Gaussian分布变成了Mixture Gaussian机制，原本的worst case在Mixture Gaussian分布下出现的概率被平均分到了不同的shuffle结果下。因此原本worst case距离被拉近，差分隐私的隐私增强了。从图1我们可以看到一个明显的结论，在进行了shuffle之后的Gaussian分布明显比之前的更加平滑。因此在计算worst case的概率的距离时一定会比之前的更小，这在后面的章节中会通过理论分析证明这一点。

The primal challenge hence lies in \textit{quantifying the privacy guarantee brought by random shuffling of model weights}. According to previous study \cite{wang2024unified}, random shuffling introduces mixture distributions of which the tradeoff function is lower bounded by the unshuffled case, indicating higher  indistinguishability of the former. However, the former works do not provide an explicit form of the privacy guarantee achieved by shuffling. To vividly explain why mixture distributions increase indistinguishability, we will show a toy example by mixture of Gaussians.

To better illustrate our intuition, we employ the Gaussian distribution as a toy example in Fig.~\ref{fig:toy_example}. We have two points $(-2,0)$ and $(2,0)$ as the original data, the additive Gaussian mechanism generates outputs following distribution $\mathcal{N}((-2,0), I)$ and $\mathcal{N}((2,0), I)$ (left panel), respectively, providing plausible deniability in discriminating the two points. The closer the distance between the two distributions, the harder it is to discriminate. What shuffling does is to draw the two distributions closer by randomly permuting the data, i.e., the original $(2,0)$ has a half of the chance to be revealed as $(0,2)$. With the noise, the distribution $\mathcal{N}((2,0), I)$ morphs into a Gaussian mixture centered at $(2,0)$ and $(0,2)$, and a Gaussian mixture is created for $\mathcal{N}((-2,0), I)$ likewise. The measured distribution-wise distance is shorter between the Gaussian mixtures than that of the two original Gaussians, thereby enhancing the indistinguishability. 

\begin{figure}
	\centering
	\includegraphics[width=0.95\linewidth]{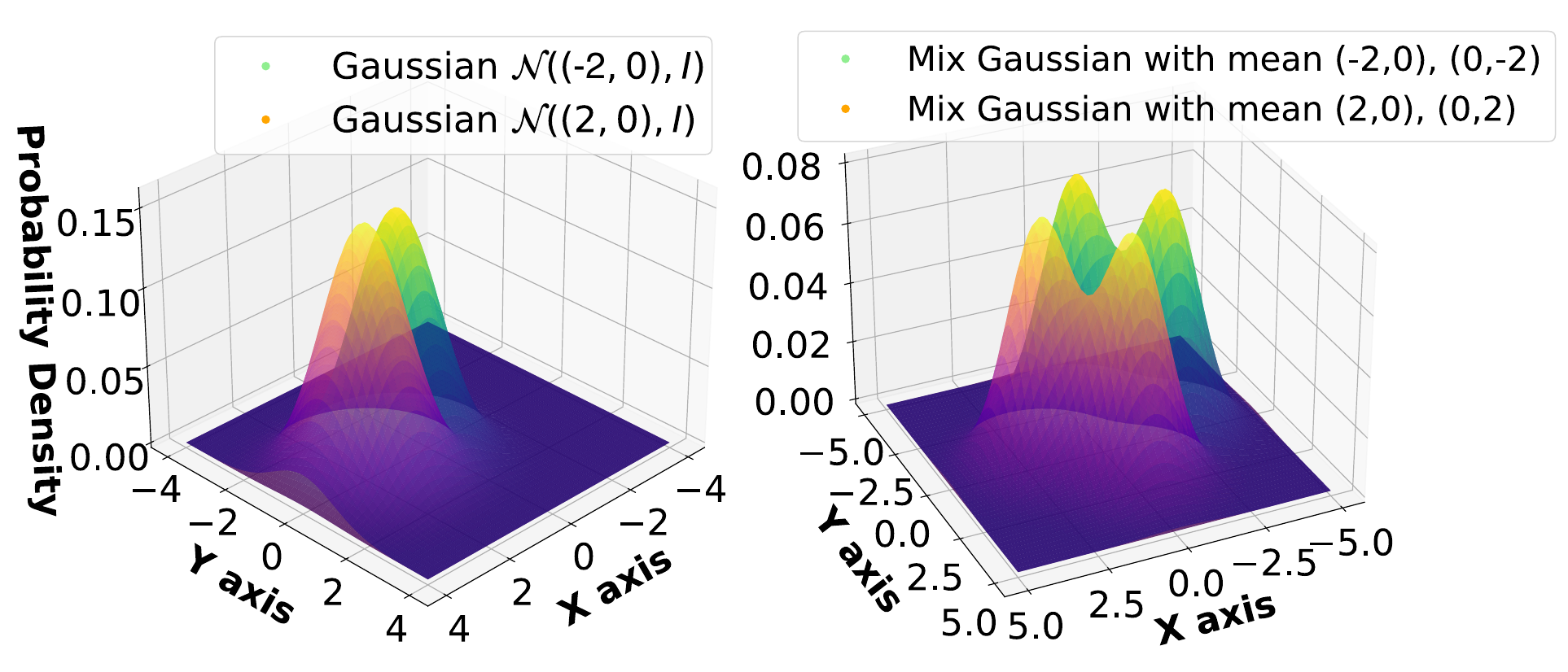}
	\caption{Left: two Gaussian distributions. Right: the distributions after random shuffling. The $\mathcal{N}((-2,0), I)$ turns into a Gaussian mixture distributed around $(-2,0)$ and $(0,-2)$ and similarly for $\mathcal{N}((2,0), I)$. We estimate the distributional distances by $\|P(x,y) - Q(x,y))\|_F$, where $\|\cdot\|_F$ means Frobenius norm, on an area $ x,y \in [-10, 10]$. The measured distance on the left panel is 18.53 which is larger than 13.96 of the right panel.}
	\label{fig:toy_example}
\end{figure}

Yet, we encounter another challenge in accounting the privacy loss of the Gaussian mixtures which contain sum of log-normal distributions without a closed-form expression. We tackle the challenge by approximating the sum of log-normal distributions and provide an analytical condition to meet $(\varepsilon, \delta)$-DP. Based on the condition, we propose the \textit{shuffled DPSGD} mechanism, which takes one more step in each iteration of DPSGD: the gradients and weights of MLP and Transformer encoder blocks get randomly shuffled in each update. We compose the privacy loss over iterations by the advanced composition and amplify privacy by subsampling \cite{Steinke2022CompositionOD}, as other accountant methods including moments accountant \cite{abadi2016deep}, RDP \cite{mironov2019renyi}, GDP \cite{dong2021gaussian} have limitations on noise distributions apart from Gaussian.

We verify the permutation invariance property, and compare the performance of shuffled DPSGD with baselines on a variety of the state-of-the-art (large) models and tasks. We also audit our mechanism by the current auditing tool, showing that it indeed meets the privacy guarantee proposed. Experimental results demonstrate that the shuffled DPSGD leads to superior model accuracy to baselines at the same privacy level, especially when the privacy budget is tight (low $\varepsilon$s). %On some datasets and large models, its performance is even close to the non-private case. 
From the implementation perspective, our shuffling mechanism incurs little overhead,  readily to be deployed in the current DPSGD framework.

Highlights of our contribution are as follows. First, we theoretically prove that permutation of the outcome enhances the privacy guarantee, and propose the approximated condition for $(\epsilon, \delta)$-DP to hold in the shuffling case. Second, based on our theory and the permutation invariance property of neural networks, we propose \textit{shuffled DPSGD} which improves the privacy guarantee over the conventional DPSGD. Finally, we verify our theory through experiments and show the efficacy of our design in the real-world large model private training.

%% file: arxiv/related.tex
\section{Related work}
 Differentially-Private Stochastic Gradient Descent (DPSGD) is a method for protecting training data using differential privacy in machine learning \cite{song2013stochastic,abadi2016deep}. We will discuss the related works w.r.t. DPSGD.

\textbf{DPSGD for natural language processing (NLP) tasks.} Applying DPSGD to large models is hard in that the high dimensionality of the model often results in excessive amount of noise, leading to reduced model performance. Consequently, existing works have explored various approaches to enhance the accuracy of large models under differential privacy protection. For instance, studies \cite{li2022large,Anil2021LargeScaleDP,hoory-etal-2021-learning,yu2022differentially,he2023exploring} have focused on NLP tasks closely associated with large models. To reduce the memory costs of DPSGD in training large models, Li et al. introduced `ghost clipping' \cite{li2022large}, which significantly saves memory by fine-tuning the weight coefficients without instantiating each sample gradient. He et al. implemented a grouped and layered clipping mechanism \cite{he2023exploring}, reducing the memory overhead of DPSGD while achieving differentially-private training of GPT-3 model (with 175B parameters). Hoory et al. designed a differentially-private vocabulary algorithm \cite{hoory-etal-2021-learning}, allowing training on a customized, domain-specific vocabulary while maintaining privacy. Yu et al. adopted classic compression methods in natural language processing (such as LoRA, Adapter, and Compacter) as `plug-ins' \cite{yu2022differentially}, efficiently perturbing the compressed module instead to avoid overwhelming noise. By appropriately configuring hyperparameters such as batch size, learning rate, etc., Anil et al. conducted private pre-training \cite{Anil2021LargeScaleDP}, improving accuracy on masked language model BERT-Large (with approximately 340M parameters). These methods alleviate the accuracy degradation issue but alter either the original training process, or the original model, or the hyperparameter setting.

\textbf{DPSGD for computer vision (CV) tasks.} Studies \cite{Kurakin2022Toward,yu2021large,Golatkar2022MixedDP,de2022unlocking,bu2022scalable} have primarily focused on large models in image processing tasks. Bu et al.'s work \cite{bu2022scalable} extended the `ghost clipping' \cite{li2022large} to convolutional neural networks (CNNs) like ResNet, reducing the computational overhead of DPSGD for CNNs. Another series of works focus on model reparametrization. Yu et al. decompose each weight matrix into the product of two low-rank matrices and a residual matrix \cite{yu2021large}, reducing the overall noise by perturbing only the gradients of the low-rank matrices. The noisy low-rank matrix gradients are then projected back to update the original weight matrices. However, this method still incurs high costs in weight reparametrization, and the noise could be amplified in projection, leading to significant accuracy losses. Golatkar et al. proposed an adaptive differential privacy algorithm \cite{Golatkar2022MixedDP}, balancing data privacy and model accuracy by compressing private data gradients guided by public data gradient directions. Although such model compression-based works have improved efficiency/accuracy to some extent, they have deviated from the original large model training. In terms of hyperparameter configuration, De et al. found that the choice of hyperparameters in DPSGD is crucial to the training outcome \cite{de2022unlocking} that, by selecting proper hyperparameters, their work achieves over 80\% accuracy on the fine-tuned pre-trained NFNet-F3 model on ImageNet. Similarly, Kurakin et al. improve accuracy of the deep ResNet model on ImageNet \cite{Kurakin2022Toward} by tuning hyperparameters.

\textbf{Shuffle-DP.} 
% 这里简单介绍一下Shuffle-DP的工作，主要是为了区分我们的工作和现有的shuffle-DP的工作。从Encode, Shuffle, Analyze (ESA)这个工作的提出，Shuffle DP开始被引入Local-DP中，在一个可信的第三方aggregator中对不同用户的数据进行shuffle，为数据添加匿名性以增强隐私保护的效果。在采用shuffle-DP优化可以将原本满足eps-LDP的机制变成满足eps‘，delta-DP的机制（eps>eps'）。并且现有的工作也在不断优化shuffle产生的eps的bound。因此进行shuffle-DP优化是一个在LDP中常见的方法。但是，shuffle-DP的方法主要是针对LDP的，而我们的工作是针对DPSGD的，并不需要可信的第三方aggregator，且我们对梯度的维度进行shuffle，而不是和前文工作是对数据集中的不同样本进行shuffle。尽管我们也通过shuffle带来的随机性优化了隐私参数，但我们的工作和shuffle-DP的工作区别依然非常大。
Since the proposal of the Encode, Shuffle, Analyze (ESA) system \cite{bittau2017prochlo}, the previous line of Shuffle-DP \cite{Erlingsson2019Amplification,cheu2019distributed,balle2019privacy} has mainly focused on reducing the privacy budget of local differential privacy, where private data from different users gets shuffled by a trusted third-party aggregator, adding anonymity upon the original DP protection. That is to say, even in cases where the true value is uploaded, the adversary cannot pinpoint which one is the true value. Specifically, the anonymity brought by shuffling turns the original $\varepsilon$-DP into $(\varepsilon^{\prime}, \delta)$-DP with $\varepsilon > \varepsilon^{\prime}$, and such a privacy loss bound is further refined by works such as \cite{feldman2021hiding}.

Our work follows a similar idea but takes a totally different approach. While Shuffle-DP primarily targets LDP, our work is designed for DPSGD and does not require a trusted third-party aggregator. We permute the model weights instead of shuffling the dataset to get different sets of samples \cite{Kairouz2021Practical}. Moreover, the permutation brings additional computational issue in model training scenarios.

% 通过对经过LDP扰动后的数据会有一定的概率返回真实数据，一定的概率返回随机值。但如果这些数据经过可信第三方进行shuffle，那么攻击者即使知道某位用户上传的真实值，依然无法确定该真实值的位置。因此，shuffle-DP的方法可以增强LDP的隐私保护效果。

%% file: arxiv/prelim.tex
\section{Preliminaries}
This section introduces the basic concepts and definitions fundamental to our study. 

\subsection{Differential Privacy}
Differential Privacy (DP) is a framework for quantifying the privacy guarantees offered by a randomized mechanism. The definition of $(\epsilon, \delta)$-differential privacy provides a measure of the likelihood that an individual's data can influence the output of an algorithm.

\begin{definition}[\textbf{$(\varepsilon,\delta)$-Differential Privacy} \cite{dwork2006calibrating}]
 \label{def:dp}
 A randomized mechanism $M$ satisfies $ (\varepsilon,\delta) $-differential privacy if for any neighboring datasets $ {X} $ and $ {X}^{\prime} $ differing by at most one unit, and for any possible output $\mathcal{O}$,
 \begin{equation} \label{eq:dpeq}
 \Pr (\mathcal{M}({X}) \in \mathcal{O}) \le e^{\varepsilon} \Pr (\mathcal{M}({X}^{\prime}) \in \mathcal{O} ) +\delta.
 \end{equation}
\end{definition}
In the special case of $ \delta = 0 $, we call $ \mathcal{M} ~ \varepsilon $-differentially private.

We define $ l_{2} $-sensitivity on a pair of adjacent datasets which differ by a single record:
\begin{definition}[$ \ell_{2} $-sensitivity]
	\label{def:l2sensitivity}
	The $ \ell_{2} $-sensitivity of the query function $ f({X})$ is defined as
	$$
	s_{2}(f) = \underset{d({X},{X}^{\prime}) = 1 }{\sup} \| f({X}) - f({X}^{\prime}) \|_{2},
	$$
	where $ \| \cdot \|_{2} $ is the $\ell_{2}$ norm. 
\end{definition}

\textbf{DPSGD.} In a deep learning task, the sensitive training dataset $X=[x_1,x_2,\ldots,x_N ]$ requires to be protected in $ T $ iterations of stochastic gradient descent. In each iteration, a batch of data $B$ of size $|B|$ will be randomly selected to compute the gradient for weights $W: g=1/|B| \sum_{x\in B} g(W,x)$, where $g(W,x)$ represents the gradient of the individual record $x$ and the single $g$ represents the average gradient of the batch. Since DP requires that the sensitivity of the outcome is bounded, conventional DPSGD conducts per-sample clipping on the gradients of each $x$:
\begin{equation}
	\bar{g}(W,x)=g(W,x)/ \max(1,\frac{\|g(W,x)\|_2}{C}),
\end{equation} 
to ensure that the sensitivity of the gradient, or clipping value is $C$, and thus the clipped batch gradient becomes $\bar{g}=1/|B| \sum_{x\in B}\bar{g}(W,x)$. DPSGD mechanism $\mathcal{M}$ inserts noise to the batch gradient: $\mathcal{M}(\bar{g} )=\bar{g}+ \frac{1}{|B|}z,$ where $  z \sim \mathcal{N}(0,\sigma^2 C^2 I) $ has the same shape with $\bar{g}$. The standard deviation $ \sigma $ of the noise is a constant decided by the privacy budget $(\epsilon,\delta)$ and $C$.

\subsection{Model Structure}
%这里我们对于排列不变性进行一个详细得叙述，主要需要是针对大模型中常用到的层，例如线性层和Multi-Head Attention层进行证明，保证在shuffle后，整个模型的输出结果不变。
Here we give a detailed description of the common layer structures in large models, including MLP and Transformer encoder blocks.

\begin{definition}\label{def:linear_layer}
  We define MLP with two-layer linear projection as 
    \begin{align}
        & x_{t+1} = h(x_{t} W_{t}^{\top} + b_{t}), \label{eq:lineareq1}\\
        & x_{t+2} = h(x_{t+1} W_{t+1}^{\top} + b_{t+1}), \label{eq:lineareq2}
    \end{align}
    where $W_t$ and $b_t$ is the weight and bias of $t$-th linear layer, $W_{t+1}$ is the weight and bias of $(t+1)$-th linear layer and $h(\cdot)$ is the activation function. 
\end{definition}

\begin{definition}\label{def:multihead_attention}
Standard transformer encoder block \cite{Vaswani2017AttentionIA} is a popular building block for neural architectures, which contains multi-head attention and a MLP layer:
\begin{align}
&\text{Attention}(Q, K, V)=\operatorname{softmax}\left(\frac{Q K^T}{\sqrt{d_k}}\right) V, \\
    &X_{t+1}=\left[ A_1, \ldots, A_h \right] W^O, \\
    &\text { where } A_{i} = \text{Attention} \left(X_t W_i^Q, X_t W_i^K, X_t W_i^V\right),~i \in [h].
\end{align}
The projections are parameter matrices $W_i^Q \in \mathbb{R}^{d_{m} \times d_k}, W_i^K \in \mathbb{R}^{d_{m} \times d_k},$ $ W_i^V \in \mathbb{R}^{d_{m} \times d_v}$ and $W^O \in \mathbb{R}^{h d_v \times d_{m}}$. $W^O$ means the output weight matrix, and $d_m, d_k, d_v$ are the dimensions of the model, key, and value, respectively.
\end{definition}
We will prove the parameter permutation invariance under the shuffle mechanism on these two types of layers.

\subsection{Sum of Log-Normal Distribution}
\begin{lemma}[F-W approximation of sum of lognormal distributions \cite{Fenton1960}]\label{lemma:lognormal}
Given $d$ i.i.d. normal distributions $Y_i ~\sim \mathcal{N}(\mu_i, \sigma^2)$ for $i=1, \ldots, d$. The sum of $d$ log-normal random variables $S_d = \sum\limits_{i=1}^d e^{Y_i}$ can be approximated by:
\begin{equation}
    S_d \approx e^Y,~ \text{where}~ Y \sim \mathcal{N}(\mu_Y, \sigma_Y^2),
\end{equation} 
given that
\begin{align}
\sigma_Y^2& = \log \left[(e^{\sigma^2}-1) \frac{\sum\limits_{i=1}^d e^{2\mu_i}}{(\sum\limits_{i=1}^d e^{\mu_i})^2} + 1 \right],\\
\mu_Y &= \log \left[\sum\limits_{i=1}^d e^{\mu_i}\right] + \frac{\sigma^2}{2} - \frac{\sigma_Y^2}{2}.
\end{align}
\end{lemma}

\subsection{Privacy Accountant}
To track the privacy loss in DPSGD, Moments Accountant \cite{abadi2016deep} and Renyi Differential Privacy Accountant \cite{mironov2019renyi} are commonly used. These methods accumulate the logarithm of the moments of privacy loss at each step, but have the drawback of overestimating the privacy budget \cite{pld1}. 
Therefore, Sivakan et al. proposed the Privacy loss Random Variables (PRVs) Accountant in \cite{PRVAccountant}, outperforming other accountant schemes in the context of DPSGD. This method truncates and discretizes the Privacy Loss Random Variables (PRVs), and then uses Fast Fourier Transform (FFT) to efficiently convolve the distributions of the PRVs. Adapted from the PLD Accountant \cite{pld1, pld2, pld3}, the PRV Accountant solves the problem in PLD Accountant that it cannot compose differentially-private algorithms. However, the PRV Accountant is still limited in that it cannot be applied to random variables of any distribution. Taking generalizability and tightness of the bound into account, we choose advanced composition \cite{Steinke2022CompositionOD} in this work:
\begin{theorem}[Advanced Composition, Thm.~22 in \cite{Steinke2022CompositionOD}]\label{thm:advanced_composition}
	For $j \in[k]$, letting $M_j: \mathcal{X}^n \times \mathcal{Y}_{j-1} \rightarrow \mathcal{Y}_j$ be the randomized algorithm which is $\left(\varepsilon_j, \delta_j\right)$-DP, we inductively define $M_{1 \cdots j}: \mathcal{X}^n  \rightarrow \mathcal{Y}_j$ by $M_{1 \cdots j}(x)=M_j\left(x, M_{1 \cdots(j-1)}(x)\right)$, where $M_{1 \ldots 0}(x)=y_0$ for some fixed $y_0 \in \mathcal{Y}_0$. Then $M_{1 \ldots k}$ is $(\varepsilon, \delta)$-DP for any $\delta>\sum_{j=1}^k \delta_j$ with
\begin{equation}
	\varepsilon=\min \left\{\sum_{j=1}^k \varepsilon_j, \frac{1}{2} \sum_{j=1}^k \varepsilon_j^2+\sqrt{2 \log \left(1 / \delta^{\prime}\right) \sum_{j=1}^k \varepsilon_j^2}\right\},
\end{equation}
where $\delta^{\prime}=\delta-\sum_{j=1}^k \delta_j$.
\end{theorem}
This Theorem provides a general composition method that can be applied to any differential privacy mechanism. It is not limited to any specific distribution, and thus we apply it in our proposed method to calculate the privacy budget. 

Moreover, we apply the Privacy Amplification by Subsampling (Thm.~29 in \cite{Steinke2022CompositionOD}) to amplify the privacy budget by batch sampling:
\begin{theorem}[Privacy Amplification by Subsampling]\label{thm:subsample}
     For dataset $x \in \mathcal{X}^n$, let $x_U \in \mathcal{X}^n$ denote the entries of $x$ indexed by $U \subset[n]$ which is a random subset. That is, $\left(x_U\right)_i=x_i$ if $i \in U$ and $\left(x_U\right)_i=\perp$ if $i \notin U$, where $\perp \in \mathcal{X}$ is some null value. Given $M: \mathcal{X}^n \rightarrow \mathcal{Y}$ satisfying $(\varepsilon, \delta)$-DP, we define $M^U: \mathcal{X}^n \rightarrow \mathcal{Y}$ by $M^U(x)=M\left(x_U\right)$.
     Let $p=\max _{i \in[n]} \underset{U}{\mathbb{P}}[i \in U]$. Then $M^U$ is $\left(\varepsilon^{\prime}, \delta^{\prime}\right)$-DP for $\varepsilon^{\prime}=\log \left(1+p\left(e^{\varepsilon}-1\right)\right)$ and $\delta^{\prime}=p \cdot \delta$.
 \end{theorem}

%Assume that, for all $i \in[n]$, we can define $U_{-i} \subset[n] \backslash\{i\}$ such that the following two conditions hold.- For all $x \in \mathcal{X}^n$ and $i \in[d], x_U$ and $x_{U_{-i}}$ are always neighbouring datasets.- For all $i \in[n]$, the marginal distribution of $U_{-i}$ conditioned on $i \in U$ is equal to the marginal distribution of $U$ conditioned on $i \notin U$.

This theorem could be applied to amplify the privacy budget of each step in DPSGD, where the random subset $U$ represents the batch sampled.

%% file: arxiv/method.tex
\section{Methodology}
%1、先介绍shuffle对于SGD的某些网络结构存在扰动不变性。
%2、给出能够降低eps，增强隐私的motivation，并且结合公式进行说明
%3、给出shuffle后，Gaussian机制能够满足的eps delta，并且给出简单的证明
%4、说明一般情况会比0向量和e1向量更差，说明其在一般情况下的适用性

This section delineates our approach. We begin with the motivation of using shuffling to reinforce privacy, illustrate our shuffling mechanism, and then provide the differential privacy guarantee of shuffled Gaussian --- a Gaussian mechanism with post shuffling.

\textbf{Motivation.}
From the entropy perspective, shuffling would increase the ambiguity of the outcome, thereby enhancing the privacy strength of the input. Specifically, shuffling increases the randomness in data processing by rearranging elements of the (noisy) output, which disrupts the potential relation between the output and its private input, therefore reducing reliance on the randomness of the additive perturbation noise. The result is a relaxed constraint on the perturbation noise at the same level of privacy guarantee.

We take the Gaussian mechanism as an example for a deeper analysis. A Gaussian mechanism composed by random shuffling is interpreted as a mixture of Gaussian, where each shuffled noisy output appears with a certain probability. Random shuffling turns a single-spike (Gaussian) distribution into a multi-spike (mixture of Gaussian) one where each spike has a reduced magnitude compared to the single spike. Thus the transformation from a Gaussian into a mixture of Gaussian can be considered as distribution smoothing, making adjacent distribution pairs harder to distinguish.

Meanwhile, it is also crucial to prevent shuffling to affect the actual SGD procedures: no accuracy price should be paid for shuffling. We will show how we achieve this in the following section.

%为了防止模型泄露个体数据的隐私信息，差分隐私算法在每次迭代中对待更新的梯度进行梯度逐样本限幅，噪声添加，以此来保障数据隐私。在高维差分隐私机制中，为了保护隐私，高维数据通常需要较大幅度的噪声扰动。然而，过度的噪声扰动可能会对模型性能产生不利影响。为了解决这一问题，我们引入了重排扰动机制。这种机制的核心优势在于，它能够在不影响模型预测结果的前提下，巧妙地引入扰动。具体来说，重排机制通过重新排列数据集中的元素，增加了数据处理过程中的随机性，这不仅有效地混淆了数据之间的潜在关联，从而增强了隐私保护，而且还显著减少了对传统噪声扰动的依赖。这种减少依赖的结果是噪声方差的显著降低，进而提高了模型的整体性能和数据的可用性。这一策略不仅保持了数据的隐私性，同时也显著提升了模型的整体性能。
% In order to prevent the model from leaking the privacy information of individual data, the differential privacy algorithm limits the gradient to be updated by sample-by-sample, adds noise, and thus protects the data privacy. In high-dimensional differential privacy mechanisms, in order to protect privacy, high-dimensional data usually requires a large amount of noise perturbation. However, excessive noise perturbation may have an adverse effect on the model performance. To address the first challenge from Sec. Intro, we introduce the shuffle mechanism.

%由于shuffle机制，我们将输出的Gaussian机制变成了Mixture Gaussian Distribution，从宏观的角度看，这使得分布变得更加平滑。由于Mixture Gaussian Distribution是对d个概率密度函数的平均，原本的高斯分布都是单峰的分布，那么平均过后，这些峰依然会存在，但其高度一定大幅度下降。因此这个分布的最大值和最小值之间的距离变小了，可以说Mixture Gaussian Distribution的方差一定比之前更大。更大的方差一定是比之前的Gaussian分布提供了更强的扰动能力，因此我们可以在Mixture Gaussian Distribution的基础上提供更强的隐私保护能力。

\subsection{The Shuffling Mechanism}

Although the shuffling of the outcome would bring additional randomness, we need to carefully implement it to avoid any model performance decline. Inspired by the permutation equivariance property \cite{xu2024Permutation}, we propose a weight shuffling mechanism that does not alter the forward and backward propagation. Our definition of permutation invariance is as follows:

%最好的情况是，shuffle机制不影响模型的输出结果。为此应对挑战二，我们精心设计了shuffle机制，可以保证不影响模型输出的结果，称之为置换不变性。我们首先给出置换不变性的定义。

%这里我们首先给出Parameter Permutation Invariance的定义，
\begin{definition}[Weights Permutation Invariance]\label{def:perin}
    %对于一个模型得参数W，给定一个shuffle机制使得改模型满足。。，则我们称该模型满足shuffle机制下得参数不变性。
    For a model $f(W, x)$ with weights $ W $ and input $x$, the permutation-invariant set $\mathcal{P}_s$ is defined as 
    \begin{equation}
        \mathcal{P}_s = \{ P ~| f(W, x) = f(P(W), x), \forall x \}
    \end{equation}
    where the permutation $P$ denotes a mapping from a set to itself with a permuted order. 
\end{definition}

Provided the definition, we now show two widely used network structures are permutation-invariant.

%这里我们提出了一种基于shuffle的差分隐私机制。和传统的local-DP中的shuffle机制不同，这里我们shuffle的对象是加噪后的输出，也就是SGD过程中的梯度和参数。然而由于不同网络结构的不同，我们给出了在现有的两个应用广泛的网络结构中的shuffle机制满足Permutation Invariance的定理。
%Here we propose a shuffled-DPSGD. Unlike the traditional shuffle mechanism in local-DP, here we shuffle the object of the noise-added output, i.e., the gradient and parameter in DPSGD. However, due to the different network architectures, we give theorems for shuffle mechanisms in two existing widely used network architectures that satisfy Permutation Invariance.

\begin{theorem}[MLP Permutation Invariance] \label{thm:linear}
    %对于线性层，我们现在可以很方便得通过对前后两层同时进行shuffle操作，从而得到shuffle不变得结果。
    Let $W = \{W_t, W_{t+1}, b_t, b_{t+1}\}$ be the parameters of MLP defined in Def.~\ref{def:linear_layer}, and $x = x_t$ be the input to the MLP. For any element-wise activation function $h(\cdot)$ like ReLu, tanh and Softmax, MLP $f(W, x)$ is permutation-invariant in both the forward and backward propagation if the following permutation is applied:
    \begin{align}
        P(W)=\{P_{1}^{\top} W_t, ~b_t P_1,~W_{t+1} P_1 ,~ ~b_{t+1}\},
    % \begin{cases}
    %     & P^{\top} W_t, ~b_t P,\\
    %     & W_{t+1} P,~ ~b_{t+1}.\\
    % \end{cases}
    \label{eq: linear permutation}
    \end{align}
where $P_1$ is any permutation matrix.
\end{theorem}
The proof can be obtained by calculation, and we place it in Appendix~\ref{proof:theorem1}.

\begin{theorem}[Transformer Permutation Invariance] \label{thm:attention}
    %根据通用性，我们这里主要针对在Transformer中出现的Attention-based architectures进行保证。对于attention层，我们也可以通过对KQV层以及后面的MLP层，同时进行shuffle操作，从而得到在attention层的shuffle不变得结果。
    Let $W = \{W_i^K, W_i^Q, W_i^{V}, W^{O}\}$ for $i \in [h]$ with $h$ heads be the parameters of the transformer block defined in Def.~\ref{def:multihead_attention}. For any element-wise activation function $h(\cdot)$ like ReLu, tanh and Softmax, the transformer block is permutation-invariant in both the forward and backward propagation if the following permutation is applied:
    \begin{equation}
        P(W) = \{ W_i^K P_{1i}, ~W_i^Q P_{1i}, ~W_i^{V} P_2, ~P_3 W^{O}\},
    \end{equation}
 where $P_{1i}\in \mathbb{R}^{ d_k \times d_k}, P_{2} \in \mathbb{R}^{ d_v \times d_v},$ and
    \begin{equation}
        P_3 = 
        \begin{bmatrix}
        P_2^{\top} & 0 & \cdots & 0 \\
        0 & P_2^{\top} & \cdots & 0 \\
        \vdots & \vdots & \ddots & \vdots \\
        0 & 0 & \cdots & P_2^{\top}
        \end{bmatrix} \in \mathbb{R}^{h d_v \times h d_v}
    \end{equation}
   are any permutation matrices.
\end{theorem}

We also refer readers to the proof in Appendix~\ref{proof:theorem2}.
%这里我们的shuffle机制的不变性主要针对上述的两种网络结构，尽管这里只有两个网络结构，但是由于transformers在最近网络结构中的大幅度运用，这两个网络结构的参数量占据了大部分的网络参数的绝大部分，因此我们的shuffle机制在实际中具有很强的通用性。

Although we merely define permutation invariance on two network structures, the two structures widely exist in all types of neural networks. Particularly, the transformer block constitutes a major component of the popular Transformer-based large models. Hence a vast majority of the current large models could permute all or a part of their parameters without altering their outputs by Thm.~\ref{thm:linear} and \ref{thm:attention}. 

\textbf{Randomized shuffling.} Since any permutation from $\mathcal{P}_s$ can be chosen to apply to the model weights without changing the output, we sample a permutation across all elements in $\mathcal{P}_s$ by a uniform distribution. Therefore, all analysis is based on that the permutation is \textit{uniformly} sampled from $\mathcal{P}_s$ in the following. We start from a simple case of shuffling a vector to analyze how random shuffling impacts the output distribution, with Gaussian mechanism as an instance.

\iffalse
We consider a probabilistic space on the permutation
Since the permutation order is randomly determined, we hereby define the probability space on $\mathcal{P}_s$:
\begin{definition}
	Given the permutation-invariant set $\mathcal{P}_s$, we define a probability space on this set. Let $(\Omega, \mathcal{F}, \Pr)$ be a probability space where:
	\begin{itemize}
		\item $\Omega = \mathcal{P}_s$ is the set of all possible permutations.
		\item $\mathcal{F}$ is a $\sigma$-algebra on $\Omega$, representing the set of all possible events.
		\item $\Pr$ is a probability measure defined on $\mathcal{F}$, assigning probabilities to events.
	\end{itemize}
\end{definition}
We specifically adopt from the probability space a particular distribution where all possible permutations from $\mathcal{P}_s$ appear with equal likelihood.
\fi

We discuss how shuffling is integrated with a vectorized Gaussian mechanism. In the conventional Gaussian mechanism, a zero-centered Gaussian noise vector is added to the output vector (of the same size) to achieve DP. Our shuffled Gaussian mechanism only takes one more step.

\begin{definition}[Shuffled Gaussian]\label{def: shuffle mechanism}
    A uniformly chosen random permutation $P \in \mathcal{P}_s$ turns vector $y \in \mathbb{R}^d$ into vector $P(y) \in \mathbb{R}^d$. The shuffled Gaussian is defined as $\mathcal{M}(y) = P(y)+z$, where $z \sim \mathcal{N}(\bm{0}, \sigma^2 I) \in \mathbb{R}^d$. Letting the outcome be the random variable $o = P(y)+z ,$ its probability density function (PDF) is 
    \begin{align}\label{eq:proboutput}
        p(o) = \frac{1}{|\mathcal{P}_s|} \sum_{i=1}^{|\mathcal{P}_s|} \frac{1}{(2\pi \sigma^2)^{|\mathcal{P}_s|/2}} e^{-\frac{\|o-y^{(i)}\|^2_2}{2\sigma^2}},
    \end{align}
where $y^{(i)}$ is the $i$-th permutation of $y$ and $|\mathcal{P}_s|$ denotes the total number of permutations which meet Def.~\ref{def:perin}.
\end{definition}
%这里我们所述的shuffle机制是对差分隐私部署的机制而言，但是对于shuffled-DPSGD而言，我们的shuffle机制是对Gaussian机制的输出进行shuffle。但由于梯度是需要对应加到参数上的，因此我们需要对参数进行shuffle，从而得到shuffle不变的结果。因此我们这里的shuffle机制是对梯度和参数同时进行shuffle。
Without causing confusion, we use $P$ to both denote a deterministic permutation as well as a random variable on $\mathcal{P}_s$. In the shuffled Gaussian mechanism, the original output vector is first rearranged by a permutation which is randomly sampled from $\mathcal{P}_s$, and then is added a zero-centered Gaussian noise vector. Notably, the PDF of Eq.~\ref{eq:proboutput} remains the same if we choose to add noise before shuffling the noisy vector. This is because the Gaussian noise $z$ is i.i.d. in each dimension. 

%Here we describe the shuffle mechanism for the deployment of differential privacy, but for shuffled-DPSGD, our shuffle mechanism is to shuffle the output of the Gaussian mechanism (gradient). However, since the gradient needs to be added to the corresponding parameters, we need to shuffle the parameters to get the result of Permutation Invariance. Therefore, our shuffle mechanism here is to shuffle the gradient and parameters at the same time.

\subsection{Shuffling Improves Privacy}\label{sec:improveprivacy}
A vital point in our motivation is that the output distribution gets smoother and the distribution-wise distance gets shorter by the shuffling mechanism. In shuffled Gaussian, we consider the shuffling mechanism a post-processing step of the Gaussian mechanism, and the shuffling mechanism does not access the original data. Therefore, the shuffled Gaussian must at least satisfy the DP that the pure Gaussian mechanism provides. In this section, we give the proof that shuffled Gaussian indeed improves the privacy guarantee of the Gaussian mechanism.

\begin{theorem}[Comparison with Gaussian mechanism]
	The privacy guarantee $\delta$ of the Shuffled Gaussian is smaller than or equal to that of the Gaussian mechanism under the same $\epsilon$ for meeting $(\epsilon, \delta)$-DP.
\end{theorem}
\begin{proof}
	According to Def. \ref{def: shuffle mechanism}, the output distributions of the queries on $x$ and $x'$ are
	\begin{align}
		P\left[f(x)\right] + z  &\triangleq Y = \begin{cases} Y^{(i)}~~ w.p.~ \frac{1}{a}, \end{cases}\\
		&Y^{(i)} \sim \mathcal{N}(f^{(i)}(x), \sigma^2 I), i \in [a], \\
		P\left[f(x^{\prime})\right] + z &\triangleq Y^{\prime} = \begin{cases} Y^{\prime (i)}~~ w.p.~ \frac{1}{b},\end{cases}\\
		& Y^{\prime (i)} \sim \mathcal{N}(f^{(i)}(x^{\prime}), \sigma^2 I), i \in [b]. 
	\end{align}
	We use $a, b$ to denote the number of different permutations of $f(x), f(x')$, respectively. Let $q(y) = \frac{1}{(2\pi \sigma^2)^{d/2}} e^{- \frac{||y||^2}{2\sigma^2}}$ be defined as the probability density function (PDF) of the $d$-dimensional Gaussian distribution $\mathcal{N} (0,\sigma^2I_d)$. Then the probability density functions of $Y$ and $Y^\prime$, denoted as $p(y)$ and $p^\prime(y)$, can be written as follows:
	\begin{equation}
		\begin{aligned}
			p(y) =& \frac{1}{a} \sum\limits_{i=1}^{a} q(y-f^{(i)}(x))\\
			=& \frac{1}{a} \sum_{i=1}^{a} \frac{1}{(2\pi \sigma^2)^{d/2}} e^{-\frac{\|y-f^{(i)}(x)\|^2_2}{2\sigma^2}}             \triangleq \frac{1}{a} \sum\limits_{i=1}^{a} q_i(y),\\
			p^\prime(y) =& \frac{1}{b} \sum\limits_{i=1}^{b} q(y-f^{(i)}(x^\prime))
			% =&\frac{1}{b} \sum_{i=1}^{b} \frac{1}{(\sqrt{2\pi \sigma^2})^d} e^{-\frac{\|y-f^{(i)}(x')\|^2_2}{2\sigma^2}}\\
			\triangleq \frac{1}{b} \sum\limits_{i=1}^{b} q_i^\prime(y),
		\end{aligned}
	\end{equation}
	Here, we define $q_i(x) = q(y - f^{(i)}(x))$ and $q_i^\prime(x) = q(y - f^{(i)}(x^\prime))$. For any adjacent datasets, the query outputs satisfying differential privacy is equivalent to having
	\begin{align}
		&\Pr\left( P\left[f(x)\right] + z \in O \right) - e^{\varepsilon} \Pr\left( P\left[f(x^{\prime})\right] + z \in O \right) \leq \delta  \notag\\
		&\Leftrightarrow \delta = \sup_{O} \int_{O} p(y) - e^\epsilon p^{\prime}(y) \, \textrm{d} y, \label{eq:dpiff}\\
		&\Leftrightarrow  \delta =  \int_{O^*} \frac{1}{a}\sum_{i=1}^{a}q_{i}(y) - e^\epsilon \frac{1}{b}\sum\limits_{i=1}^{b} q_i^\prime(y)\, \textrm{d} y \label{eq:ostar}
	\end{align}
	where $ O^* = \{y \mid p(y) - e^\epsilon p^\prime (y) \geq 0\}$. Letting $k$ be the least common multiple of $a$ and $b$, we have
	\begin{equation}
		\begin{aligned}
			\delta &= \int_{O^*} [\frac{1}{k} \sum_{i=1}^{k} q_i(y)  - e^\epsilon \frac{1}{k} \sum_{i=1}^{k} q_i^\prime(y)] \textrm{d} y\\
			&= \int_{O^*} \frac{1}{k} \sum_{i=1}^{k} [q_i(y)  - e^\epsilon q_i^\prime(y)]  \textrm{d} y \\
			&= \frac{1}{k} \sum_{i=1}^{k} \int_{O^*} [q_i(y)  - e^\epsilon q_i^\prime(y)]  \textrm{d} y\\
			&\leq \frac{1}{k} \sum_{i=1}^{k} \int_{O_i^*} [q_i(y)  - e^\epsilon q_i^\prime(y)]\textrm{d} y \\
			& \leq \frac{1}{k} \sum_{i=1}^{k} \delta_0 = \delta_0.
		\end{aligned}
	\end{equation}
	$O_i^*$ is defined as $\{y|q_i(y)  - e^\epsilon q_i^\prime(y) \geq 0\}$ for all $i$, and $\delta_0$ is the largest $\delta$ among all pure Gaussian mechanisms defined on $O_i^*$, which are the same in the unshuffled case. Hence it is evident that the privacy guarantee improves by shuffling.
\end{proof}

\subsection{Privacy Guarantee of Shuffled Gaussian}
Although Sec.~\ref{sec:improveprivacy} provides evidence that shuffling improves the privacy guarantee, it is still unclear how much improvement presented and how one can manipulate that in practice. In this section, we offer a more precise, tighter privacy guarantee for the shuffled Gaussian:

% We formalize the reduction in $\varepsilon$ as a function of the original privacy parameter, the number of dimension $d$, and the sensitivity of the function $s_{2}(f)$, as shown in the equation below.

% \begin{equation}
% \varepsilon' = f(\varepsilon, d, s_{2}(f)),
% \label{eq:privacy_loss_reduction}
% \end{equation}
% where $d$ is the number of inputs to the mechanism, and $s_{2}(f)$ is the sensitivity of the function.

\begin{theorem}[Privacy Guarantee of Shuffled Gaussian]\label{thm:dpshuffled}
Given query function $f(x)\in \mathbb{R}^d $ on any adjacent datasets $X, X^{\prime}$, and the Gaussian noise $z \sim \mathcal{N}(\bm{0}, \sigma^2 I)\in \mathbb{R}^d$, the shuffled Gaussian $\mathcal{M}(f(x))$ satisfies $(\varepsilon, \delta)$-differentially private if
\begin{equation}
    \begin{split}
        \Phi \left(\frac{\sigma_{Z_1}}{2} + \frac{\zeta_1 - \epsilon}{\sigma_{Z_1}} \right) - e^{\epsilon}\Phi \left(\frac{\sigma_{Z_2}}{2} + \frac{\zeta_2 - \epsilon}{\sigma_{Z_2}} \right) \leq \delta, \label{eq:shuffledp}
    \end{split}
\end{equation}
where $\Phi$ is the CDF of the standard Gaussian distribution, and
\begin{equation}\label{zeta_result}
    \begin{split}
        \zeta_1 & = -\log(1 + (d-1)e^{-\frac{d c c^{\prime}}{(d-1)\sigma^2}}), \\
        \zeta_2 & = -\log(1 + (d-1)e^{-\frac{d c (c+c^{\prime})}{(d-1)\sigma^2}}) - \frac{c^2}{\sigma^2}, \\
        \sigma_{Z_1}^2 &= \log \left[\frac{ 1 + (d-1)e^{-\frac{2 d c c^{\prime}}{(d-1)\sigma^2}} }{(1 + (d-1)e^{-\frac{dc c^{\prime}}{(d-1)\sigma^2}})^2} (e^{\frac{c^2}{\sigma^2}} -1) + 1 \right], \\
        \sigma_{Z_2}^2 &= \log \left[\frac{ 1 + (d-1)e^{-\frac{2 d c (c+c^{\prime})}{(d-1)\sigma^2}} }{(1 + (d-1)e^{-\frac{dc (c + c^{\prime})}{(d-1)\sigma^2}})^2} (e^{\frac{c^2}{\sigma^2}} -1) + 1 \right],
    \end{split}
\end{equation}
where $\|f(x) - f(x^\prime)\|_2\le c, ~\|f(x)\|_2\le c^\prime$.
\end{theorem}
The detail of the computation and the proof are collected in Appendix~\ref{appendix:theorem3}. Here we present a sketch proof.
\begin{proof}
    At the beginning, a special case needs to be excluded from discussion, namely, $f(x)$ and $f(x')$ are both equal-entry vectors where random shuffling cannot properly work, i.e., the vector does not change after shuffling. The case is almost impossible to occur and is degraded to the unshuffled case. Hence we ignore the special case.
    
    We start from simplifying the criterion of $O^*$ in Eq.~\eqref{eq:ostar} as follows.
    \begin{align}
        & p(y) - e^\epsilon p^\prime (y) \geq 0 \\
        %\Leftrightarrow & \frac{\frac{1}{a} \sum_{i=1}^{a}q_{i}(y)}{\frac{1}{b}\sum_{i=1}^{b} q_i^\prime(y)} \geq e^{\epsilon} \\
        \Leftrightarrow & \frac{\frac{1}{k} \sum_{i=1}^{k} e^{-\| y - f^{(i)}(x)\|_2^2/(2\sigma^2)}}{\frac{1}{k}\sum_{i=1}^{k} e^{-\| y - f^{(i)}(x^{\prime})\|_2^2/(2\sigma^2)}} \geq e^{\epsilon},  \label{eq:lcm}
    \end{align}
    where $k$ is the least common multiple of $a$ and $b$. We define the difference between the adjacent queries as $g = f(x^\prime) - f(x)$, and for convenience, we abbreviate $f^{(i)}(x)$ as $f^{(i)}$. Then Eq.~\eqref{eq:lcm} can be written as
    \begin{align}
        % & \frac{\frac{1}{a} \sum_{i=1}^{a} e^{-\| y - f^{(i)}\|_2^2/(2\sigma^2)}}{\frac{1}{b}\sum_{i=1}^{b} e^{-\| y - f^{(i)} - g^{(i)} \|_2^2/(2\sigma^2)}} \geq e^{\epsilon} \\
        % \Leftrightarrow & \frac{\frac{1}{a} \sum_{i=1}^{a} e^{-(\| y \|_2^2+ \| f \|_2^2 - 2<y, f^{(i)}>)/(2\sigma^2)}}{\frac{1}{b}\sum_{i=1}^{b} e^{-(\| y \|_2^2+ \| f \|_2^2 + \| g \|_2^2 - 2<y, f^{(i)}> - 2<y, g^{(i)}> + 2<g^{(i)}, f^{(i)}>)/(2\sigma^2)}} \geq e^{\epsilon} \\
        % \Leftrightarrow & \frac{\frac{1}{a} \sum_{i=1}^{a} e^{<y, f^{(i)}>/\sigma^2}}{ e^{-(\|g\|^2_2+ 2<f, g>)/ (2\sigma^2) }\frac{1}{b}\sum_{i=1}^{b} e^{( <y, f^{(i)}> + <y, g^{(i)}> )/(\sigma^2)}} \geq e^{\epsilon} \\
        \frac{\frac{1}{k} \sum_{i=1}^{k} e^{<y, f^{(i)}>/\sigma^2}}{ e^{-(\|g\|^2_2+ 2<f, g>)/ (2\sigma^2) }\frac{1}{k}\sum_{i=1}^{k} e^{( <y, f^{(i)}> + <y, g^{(i)}> )/\sigma^2}} \geq e^{\epsilon}, \label{eq:22}
    \end{align}
% where $k$ is the least common multiple of $a$ and $b$.

%     \begin{lemma}
%         Assuming $a_1, a_2, \ldots, a_k$ and $b_1, b_2, \ldots, b_k$ are positive real numbers, we have
%         \begin{align}
%             \sum_{i=1}^{k} a_i\sum_{i=1}^{k} b_i \geq \sum_{i=1}^{k} a_i b_i.
%         \end{align}
%     \end{lemma}
%     \begin{proof}
%         Obviously,
%         \begin{align}
%             &\sum_{i=1}^{k} a_i\sum_{i=1}^{k} b_i = \sum_{i=1}^{k} a_i b_i + \sum_{i\neq j} a_i b_j >  \sum_{i=1}^{k} a_i b_i.
%         \end{align}
%     \end{proof}
% Substituting the above inequality into equation \ref{eq:22}, we obtain
%     \begin{align}
%         & \frac{\sum_{i=1}^{k} e^{<y, f^{(i)}>/\sigma^2}}{\sum_{i=1}^{k} e^{( <y, f^{(i)}> + <y, g^{(i)}> )/\sigma^2}} \geq \frac{1}{ \sum_{i=1}^{k} e^{(<y, g^{(i)}> )/\sigma^2}}.
%     \end{align}

    By Lemma~\ref{lemma:inequality} (with a proof in Appendix~\ref{appendix:theorem3}), a sufficient but not necessary condition for the output $y \in O^{*}$ is
    \begin{align}
        % & \frac{1}{ e^{-(\|g\|^2_2+ 2<f, g>)/ (2\sigma^2) }\sum_{i=1}^{k} e^{(<y, g^{(i)}> )/\sigma^2}} \geq e^{\epsilon} \\
        e^{ - \epsilon + (\|g\|^2_2+ 2<f, g>)/ (2\sigma^2)} \geq \sum_{i=1}^{k} e^{(<y, g^{(i)}> )/\sigma^2}.
    \end{align}
With the inequality, we can derive a lower bound for $\delta$ as 
    \begin{align}
        % & \delta =  \int_{O^*} \frac{1}{a}\sum_{i=1}^{a}q_{i}(y) - e^\epsilon \frac{1}{b}\sum_{i=1}^{b} q_i^\prime(y)\, dy \\
        \delta  = & \int_{O^*} \frac{1}{k}\sum_{i=1}^{k}q_{i}(y) - e^\epsilon \frac{1}{k}\sum_{i=1}^{k} q_i^\prime(y)\, \textrm{d} y \\
        % = & \frac{1}{k} \sum_{i=1}^{k} \int_{O^*} q_{i}(y)\, dy - e^\epsilon \frac{1}{k} \sum_{i=1}^{k} \int_{O^*} q_i^\prime(y)\, dy \\
        % = & \frac{1}{k} \sum_{i=1}^{k} \Pr [Y^{(i)} \in O^*] - e^\epsilon \frac{1}{k} \sum_{i=1}^{k} \Pr [Y^{\prime (i)} \in O^*] \\
        \geq & \frac{1}{k} \sum_{i=1}^{k} \left[ \Pr [\sum_{j=1}^{k} e^{(<Y^{(i)}, g^{(j)}> )/\sigma^2} \le C(\epsilon)] \right] \label{eq:prob1}\\
        & - \frac{1}{k} \sum_{i=1}^{k} \left[ e^\epsilon \Pr [\sum_{j=1}^{k} e^{(<Y^{\prime (i)}, g^{(j)}> )/\sigma^2} \le C(\epsilon)] \right] \label{eq:prob2}
    \end{align}
    where $C(\epsilon) = e^{ - \epsilon + (\|g\|^2_2 + 2<f, g>)/ (2\sigma^2)}$. The term $\sum_{j=1}^{k} e^{(<Y^{(i)}, g^{(j)}> )/\sigma^2}$ is a typical sum of log-normal distributions as $<Y^{(i)}, g^{(j)}> = \mathcal{N}(<f^{(i)}, g^{(j)}>, \|g\|_2^2\sigma^2)$ and such a distribution does not have an explicit probability density function. Thus we apply Lemma \ref{lemma:lognormal} to approximate it as a log-normal distribution:
    \begin{align}
        & \sum_{j=1}^{k} e^{(<Y^{(i)}, g^{(j)}> )/\sigma^2} \approx e^{Z_1},~ Z_1 \sim  \mathcal{N}(\mu_{Z_1}, \sigma_{Z_1}^2),\\
        & \sigma_{Z_1}^2 = \log \left[\frac{ \sum_{j=1}^{k}e^{2\mu_{ij}} }{(\sum_{j=1}^{k}e^{\mu_{ij}})^2} (e^{\frac{\|g\|_2^2}{\sigma^2}} -1) + 1 \right],\\
        & \mu_{Z_1} =  \log(\sum_{j=1}^{k} e^{\mu_{ij}}) 
        + \frac{\|g\|^2_2}{2\sigma^2} - \frac{\sigma_{Z_1}^2}{2}; 
    \end{align}
where $\mu_{ij} = \frac{<f^{(i)}, g^{(j)}>}{\sigma^2}$. Therefore, the probability distribution in Eq.~\eqref{eq:prob1} can be approximated as follows:
    \begin{align}
        & \Pr [\sum_{j=1}^{k} e^{(<Y^{(i)}, g^{(j)}> )/\sigma^2} \le C(\epsilon)] \approx \Pr [e^{\mathcal{N}(\mu_{Z_1}, \sigma_{Z_1}^2)} \le C(\epsilon)], \\
        & = \Pr [\mathcal{N}(\mu_{Z_1}, \sigma_{Z_1}^2) \le \log C(\epsilon)] = \Phi \left(\frac{\log C(\epsilon) - \mu_{Z_1}}{\sigma_{Z_1}} \right) \\
        & = \Phi \left( \frac{\zeta_1 - \epsilon}{\sigma_{Z_1}} + \frac{\sigma_{Z_1}}{2}\right), \text{where}~\zeta_1 = - \log(\sum_{j=1}^{k}e^{\mu_{ij}}) + \frac{<f, g>}{\sigma^2}.
    \end{align}
We define $\mu_{ij}^{\prime} = \frac{<f^{(i)}(x'), g^{(j)}>}{\sigma^2}$, and $\sigma_{Z_2}, \zeta_2$ by replacing $\mu_{ij}$ with $\mu_{ij}^{\prime}$ in $\sigma_{Z_1}, \zeta_1$, respectively.	By approximating the probability distribution in Eq.~\eqref{eq:prob2} similarly, we have 
    \begin{align}
        & \Pr [\sum_{j=1}^{k} e^{(<Y^{(i)}, g^{(j)}> )/\sigma^2} \le C(\epsilon)] \\
        &- e^\epsilon \Pr [\sum_{j=1}^{k} e^{(<Y^{\prime (i)}, g^{(j)}> )/\sigma^2} \le C(\epsilon)] \\
        % = & \Phi \left(\frac{\log C(\epsilon) - \mu_{Z_1}}{\sigma_{Z_1}} \right) - e^{\epsilon}\Phi \left(\frac{\log C(\epsilon) - \mu_{Z_2}}{\sigma_{Z_2}} \right) \\
        \approx & \Phi \left(\frac{\sigma_{Z_1}}{2} + \frac{\zeta_1 - \epsilon}{\sigma_{Z_1}} \right) - e^{\epsilon}\Phi \left(\frac{\sigma_{Z_2}}{2} + \frac{\zeta_2 - \epsilon}{\sigma_{Z_2}} \right) \triangleq h(\sigma_{Z_1}, \sigma_{Z_2}).
    \end{align}
    % where
    % \begin{align}
    %     \zeta_1 = - \log(\sum_{j=1}^{k}e^{\mu_{ij}}) + \frac{<f, g>}{\sigma^2} \le 0,\\
    %     \zeta_2 = - \log(\sum_{j=1}^{k}e^{\mu_{ij}^{\prime}}) + \frac{<f, g>}{\sigma^2} \le - \frac{\|g\|_2^2}{\sigma^2}.
    % \end{align}
    % \begin{align}
    %     & \Pr [\sum_{j=1}^{k} e^{(<Y^{\prime (i)}, g^{(j)}> )/\sigma^2} \le C(\epsilon)] \\
    %     = & \Pr [\sum_{j=1}^{k} e^{(<Y^{(i)} + g^{(i)}, g^{(j)}> )/\sigma^2} \le C(\epsilon)] \\
    %     \geq & \Pr [\sum_{j=1}^{k} e^{(<Y^{(i)}, g^{(j)}> + \|g\|_2^2)/\sigma^2} \le C(\epsilon)] \\
    %     = & \Pr [\sum_{j=1}^{k} e^{(<Y^{(i)}, g^{(j)}> )/\sigma^2} \le C(\epsilon) e^{-\|g\|_2^2/\sigma^2}].
    % \end{align}
    % 由$\sigma_{Z_1}^2$的定义下，可得
    % \begin{align}
    %     \sigma_{Z_1}^2 \le \frac{\|g\|^2}{\sigma^2}
    % \end{align}
    Finally, we estimate $\delta$ by deriving the upper bound for $h(\sigma_{Z_1}, \sigma_{Z_2})$ which is an approximated lower bound to $\delta$. To this end, we establish the following lemma by analyzing the property of $h$:
    \begin{lemma}\label{lemma:increasing}
    	Given the definitiion of $\sigma_{Z_1}^2, \sigma_{Z_2}^2, \zeta_1, \zeta_2$ and the constraints $\|g\|_2\le c$, $\|f\|_2\le c^{\prime}$, $h(\sigma_{Z_1},\sigma_{Z_2})$ achieves the upper bound at $k = d$, where
    	\iffalse
        According to the above definitions of $\sigma_{Z_1}^2, \sigma_{Z_2}^2, \zeta_1, \zeta_2$, there are constraints based on the definition.
        \begin{align}
            & \zeta_1 \le 0, ~~\zeta_2 \geq \zeta_1 - \frac{\|g\|_2^2}{\sigma^2}, \|g\|_2\le c, ~\|f\|_2\le c^{\prime}\\
            &  \log \left[\frac{ 1 }{k} (e^{\frac{\|g\|_2^2}{\sigma^2}} -1) + 1 \right]< \sigma_{Z_i}^2 < \frac{\|g\|^2_2}{\sigma^2},~~i=1,2,
        \end{align}
        Under the above constraints, $h(\sigma_{Z_1},\sigma_{Z_2}) = \Phi \left(\frac{\sigma_{Z_1}}{2} + \frac{\zeta_1 - \epsilon}{\sigma_{Z_1}} \right)- e^{\epsilon}\Phi \left(\frac{\sigma_{Z_2}}{2} + \frac{\zeta_2 - \epsilon}{\sigma_{Z_2}} \right)$ will yields the upper bound when $k = d$ and 
        \fi
        \begin{align}
            f = \sqrt{\frac{1}{d(d-1)}}c^{\prime}[(d-1), -1, \ldots, -1], \\
            g = \sqrt{\frac{1}{d(d-1)}}c[(d-1), -1, \ldots, -1].
    \end{align}
   Here $f$ and $g$ share the maximum value at the same dimension, not necessarily the first one. Thereby one can derive the condition of Eq.~\eqref{eq:shuffledp} and Eq.~\eqref{zeta_result}.
    \end{lemma}
The detailed proof of the lemma is provided in Appendix~\ref{appendix:lemma3}, which mainly takes advantage of the monotonicity of $h$.
% In the following, we use this special case of $f$ and $g$ to solve for the worst-case privacy budget of DPSGD.
\end{proof}

% The theorem is derived based on the necessary and sufficient conditions for differential privacy, and thus at $d = 1$ for Eq.~\ref{eq:shuffledp}, we obtain
% \begin{equation}\label{eq:shuffle-dp,dim1}
%     \delta(\varepsilon)=\Phi\left(-{\varepsilon \sigma}+\frac{1}{2 \sigma}\right)-e^{\varepsilon} \Phi\left(-{\varepsilon \sigma}-\frac{1}{2\sigma}\right).
% \end{equation}
% According to the definition provided in the context of Gaussian Differential Privacy (GDP) \cite{dong2021gaussian}, Eq.~\ref{eq:shuffle-dp,dim1} aligns with the iff condition of differential privacy guarantee offered by GDP. Hence, GDP can be considered a specific instance within the broader framework of our shuffled Gaussian mechanism.

Note that Thm.~\ref{thm:dpshuffled} is obtained at the worst case where the total number of invariant permutations $|\mathcal{P}_s|$ is the least (except for the case where all elements are the same). In general, to have the same $(\varepsilon,\delta)$-DP held, the shuffled Gaussian requires a declining noise magnitude $\sigma$ with the increase in the total number of permutations $k$, or the dimensionality of the vector $d$. In fact, the noise variance $\sigma^2$ in the shuffled Gaussian is approximately $O\left(\frac{1}{\log d}\right)$. The result is interesting in indicating a lower level of privacy budget $(\varepsilon, \delta)$ is required for the same amount of noise ($\sigma$) applied, illustrating shuffling indeed makes the output distributions harder to distinguish, and the higher the query dimension, the more difficult the discrimination.

In the following section, we will discussion how shuffling mechanism could be applied to DPSGD. It should be noted that although the theorems we give here is based on the vector form, the conclusion can be easily extended into matrix form by reshaping matrices into vectors with one-to-one correspondence.

% Here we use the F-W method to approximate the sum of log-normal distribution. Due to it is the only one have the closed-form expression. Therefore F-W method could suit for the PRV Accountant. We also analyses the other approximation method in Sec.~\ref{sec:exp_eps_sigma}.

%这里我们在差分隐私所满足的条件中，引入了shuffle的维度d,维度扰动的引入将大大降低了$\varepsilon$的值，从而提高了隐私保护的能力。同时，我们也可以看到，当$\sigma$越大时，$\varepsilon$的值也会越小，这也是符合我们的预期的，因为当$\sigma$越大时，我们的噪声也会越大，因此我们的隐私保护能力也会越强。我们会在后的实验部分通过数值试验清楚的看到eps数值上的差异。
%Here we introduce the dimension $d$ of the shuffle and it will greatly reduce the value of $\varepsilon$, thereby improving the ability of privacy protection. We will see the difference in the value of $\varepsilon$ clearly in the numerical experiments in the later experimental section.
%下面我们将介绍如何在DPSGD的过程中，运用我们的shuffle机制进行隐私保护。我们首先给出shuffle机制的伪代码，然后给出shuffle机制的隐私保护证明。需要说明的是，尽管我们这里给出的证明是基于向量形式的机制，但其依然能够在矩阵形式的机制中得到保证，因为矩阵形式的shuffle机制reshape成一维向量可以形成一一对应的结果。

%% file: arxiv/algorithm.tex
\section{Shuffled DPSGD}
%这一章节我们介绍shuffle-DPSGD的算法，包括shuffle机制在DPSGD的过程中的具体实现，以及shuffle-DPSGD的privacy accountant的实现细节部分。
In this section, we introduce the shuffled DPSGD mechanism, including the algorithmic detail, the privacy loss accountant, and its DP guarantee.

\subsection{The Algorithm}
Shuffled DPSGD mostly follows the framework of DPSGD but introduces shuffling to the permutation-invariant part of the neural network at each weights update. Hence the privacy of gradients or weights in each iteration is amplified through shuffling without accuracy decay. Alg.~\ref{alg: Shuffle Mechanism in DPSGD} for integrating the shuffling mechanism into DPSGD involves several key steps. The first step is to randomly initialize the model parameters $W_0$. In each iteration, a random sample batch $L_t$ is selected with probability $|B| / N$, where $|B|$ is the batch size and $N$ is the total number of samples. Then the per-sample gradients of $L_t$ are computed and clipped to ensure that the sensitivity of each gradient is $c$. This per-example gradient clipping step aligns with that of the normal DPSGD and ensures the condition of $\|f(x) - f(x^\prime)\|_2\le c$ is satisfied in Thm.~\ref{thm:dpshuffled}. The clipped gradients are accumulated by batch, and the noise is added to the accumulated gradients. An additional batch clipping (line 12) is required to meet the constraint of $\|f(x)\|_2\le c^\prime$ in Thm.~\ref{thm:dpshuffled}. 

As illustrated in the following section, we calculate the noise variance $\sigma^2$ by $|B|, c, N$, training steps $T$ and $\varepsilon, \delta$ by the specific privacy accountant we choose. Here we apply Advanced Composition and Privacy Amplification by Subsampling for the accounting. If the noise variance $\sigma^2$ is higher than the noise variance $\sigma_0^2$ calculated by DPSGD, the algorithm will fall back to the original implementation of DPSGD. Otherwise, a random shuffling step would be added upon noise addition before weights update. It is worth noting that the random shuffling does not impact the weight update as the permutation invariance property holds not only in the forward but also in the backward propagation, according to Thm.~\ref{thm:linear} and \ref{thm:attention}. 

Notably, we choose to add noise before shuffling ($P(y+z)$) instead of after shuffling ($P(y) + z$) in Alg.~\ref{alg: Shuffle Mechanism in DPSGD}. The former is more convenient in that we do not need to shuffle the corresponding weights apart from permuting the gradients. The two are equivalent as $P(y+z) = P(y) + P(z), ~P(z) \sim \mathcal{N}(\bm{0}, \sigma^2 I)$ following the same distribution of $z$. This indicates $P(y+z)$ has the same distribution with $P(y)+z$. Hence the change does not alter the privacy guarantee.

\begin{algorithm}
	\renewcommand{\algorithmicrequire}{\textbf{Input:}}
	\renewcommand{\algorithmicensure}{\textbf{Output:}}
	\caption{Shuffled DPSGD}
	\label{alg: Shuffle Mechanism in DPSGD}
	\begin{algorithmic}[1] 
		\REQUIRE (a) Training dataset $X$, (b) privacy budget $\varepsilon, \delta$, (c) model weights $W$. (d) clipping value $c$, batch clipping value $c^{\prime}$, batch size $|B|$, total $N$ examples, total $T$ training steps.
		\ENSURE Private model weights $W_{T+1}$.
		\STATE {Initialize $W_{0}$ randomly}
		\STATE{Compute $\sigma$ for shuffled-DPSGD by advanced composition and subsampling.}
		\STATE{Compute $\sigma_0$ for DPSGD by advanced composition and subsampling.}
		\FOR {$t \in[T]$}
		\STATE {Sample $L_{t} \in X$ with sampling probability $|B| / N$}
		\FOR {$i \in [L_t]$}
		\STATE {Compute $\mathbf{g}_{t}\left(x_{i}\right) \leftarrow \nabla_{W_{t}} \mathcal{L}(\left(W_{t}, x_{i}\right)$}
		\STATE {$\overline{\mathbf{g}}_{t}(x_{i}) \leftarrow {\mathbf{g}}_{t}(x_{i}) / \max \left(1, \frac{\left\|{\mathbf{g}}_{t}(x_{i})\right\|_{2}}{c}\right)$}
		\ENDFOR
		\STATE Accumulate the clipped gradients over a batch $ \hat{\mathbf{g}}_t = \sum_i^{|B|}\overline{\mathbf{g}}_{t}(x_{i})  $
		\IF{$\sigma < \sigma_0$ and $W_t$ is linear layer or Transformer block}
			\STATE Clip the batch gradient $ \overline{\mathbf{g}}_t = \hat{\mathbf{g}}_t / \max \left(1, \frac{\left\|\hat{\mathbf{g}}_t\right\|_{2}}{c^\prime}\right) $
			\STATE {Add noise $\tilde{\mathbf{g}}_{t} = \frac{1}{|B|}\left(\overline{\mathbf{g}}_t + z \right)$, where $z \sim \mathcal{N}(0, \sigma^2c^2I)$}
			\STATE{$W_{t+1} \leftarrow P \left[W_{t}-\eta_{t} \tilde{\mathbf{g}}_{t}\right] $}
		\ELSE
			\STATE {Add noise $\tilde{\mathbf{g}}_{t} = \frac{1}{|B|}\left(\hat{\mathbf{g}}_t + z \right)$, where $z \sim \mathcal{N}(0, \sigma_0^2c^2I)$}
			\STATE{$W_{t+1} \leftarrow W_{t}-\eta_{t} \tilde{\mathbf{g}}_{t} $}
		\ENDIF
		\ENDFOR
		\RETURN {$W_{T+1}$}
	\end{algorithmic} 
\end{algorithm}

\subsection{Privacy Budget Accountant}
In assessing the privacy cost of DPSGD, the Moments Accountant \cite{abadi2016deep}, and the Renyi Differential Privacy (RDP) Accountant \cite{mironov2019renyi} are widely recognized methods. These approaches accumulate the logarithm of the moments of the privacy loss variable at each iteration. However, they cannot accurately estimate the privacy budget for the shuffled DPSGD due to the incompatibility of the two accountant methods with the computation involving mixture of Gaussian. Specifically, these algorithms both require an exponential parameter, denoted as $\alpha$ in RDP and $\lambda$ in Moments Accountant. The issue emerges when $\alpha$ or $\lambda$ is a non-integer, leading to difficulties in expressing the PDF of the mixture of Gaussian. Even for integers, it demands the trouble to calculate multivariate expansions for the exponential parameter. All these complication hinders us from adopting their accountant methods.

We have also considered privacy accountant methods such as PLDs \cite{pld1} and PRVs \cite{pld2}. However, these techniques also encounter issues in computing Gaussian mixture. For example, the PRVs method optimizes the privacy budget by subsampling, but it produces results that are contradictory to the expected behavior of Gaussian mixture, yielding higher privacy costs under fixed noise. This is mainly due to the Fast Fourier Transform (FFT) transformation is designed for Gaussian distribution, not applicable for the composition of Gaussian mixture. 

% 因此，我们考虑了一种不需要在意差分隐私机制具体使用什么分布的一般性的组合算法，是传统的Advanced Composition Starting with Approximate DP。这里我们简单介绍一下我们使用的Advanced Composition的算法。
Therefore, we consider a general composition method that is irrelevant to the underlying noise distribution. Hence we adopt the Advanced Composition from Thm.~\ref{thm:advanced_composition} and the Privacy Amplification by Subsampling from Thm.~\ref{thm:subsample} together for calculating the privacy budget per training step. Specifically, we set an identical privacy budget across all training steps, i.e., $\epsilon_j = \epsilon_s, \delta_j = \delta_s, \forall j \in [k]$. The sampling rate $p = |B| / N$. By substituting $\epsilon = \epsilon_s, \delta=\delta_s$ into Eq.~\eqref{eq:shuffledp}, we can compute the noise variance $\sigma^{2}$ required for each training step under shuffling.

%% file: arxiv/experiment.tex
\begin{table*}[htbp]
	\centering	
	\caption{Model configurations and training hyperparameters. Shuffle ratio means the proportion of the weights that are shuffled by our mechanism in the model. }
	% \vspace{-3mm}
	\renewcommand\arraystretch{1.0}
	\label{tab: training settings}
	\begin{tabular}{cccccccc}
		\toprule[1.5pt]
		Model & \#Trainable Param  & Shuffle Ratio & Dataset & \#Train Samples $N$ & Learning Rate & Epochs & $\delta$ \\ \midrule
		% ResNet18 & 16746   & 0.15\% & \multicolumn{1}{|c}{\multirow{3}{*}{CIFAR10}} & \multicolumn{1}{c}{\multirow{3}{*}{50,000}} &  0.05  & 20 & $10^{-5}$ \\
        % WRN-50-2 & 32954 & 0.05\% & \multicolumn{1}{|c}{} & \multicolumn{1}{c}{} &  0.01  & 20 & $10^{-5}$\\ 
		ViT & 85.8M & 100\%  & \multicolumn{1}{|c}{CIFAR100} & \multicolumn{1}{c}{50,000} &  $2\times 10^{-3}$  & 10  & $5 \times 10^{-6}$  \\ \midrule
		\multicolumn{1}{c}{\multirow{4}{*}{\parbox{1.6cm}{\centering{BERT \\  RoBERTa}}}} & \multicolumn{1}{c}{\multirow{4}{*}{\parbox{1cm}{110M \\ 125M}}} & \multicolumn{1}{c|}{\multirow{4}{*}{\parbox{1cm}{100\% \\ 100\%}}} & SST-2 & 67,349 & $5\times 10^{-4}$  & 3  & $1/(2N)$ \\ 
		\multicolumn{1}{c}{}& \multicolumn{1}{c}{}&\multicolumn{1}{c|}{} & QQP   & 363,846 &  $5\times 10^{-4}$             & 18  &$1/(2N)$  \\ 
		\multicolumn{1}{c}{}&\multicolumn{1}{c}{}&\multicolumn{1}{c|}{} &MNLI & 392,702 &   $5\times 10^{-4}$              & 18  & $1/(2N)$  \\ 
		\multicolumn{1}{c}{}&\multicolumn{1}{c}{}&\multicolumn{1}{c|}{} & QNLI  & 108,000 &  $5\times 10^{-4}$              & 6 & $1/(2N)$   \\ \midrule
		\multicolumn{1}{c}{\multirow{2}{*}{\parbox{1.8cm}{\centering{GPT-2\\ GPT-2-large}}}} & \multicolumn{1}{c}{\multirow{2}{*}{\parbox{1cm}{124M \\0.838B}}} & \multicolumn{1}{c|}{\multirow{2}{*}{\parbox{1cm}{100\% \\100\%}}} & E2E & 42,043 &  0.002              & 10 &  $8 \times 10^{-6}$  \\
		\multicolumn{1}{c}{}&\multicolumn{1}{c}{}&\multicolumn{1}{c|}{} & DART & 60,591&  $5\times 10^{-4}$ & 15   & $10^{-5}$ \\
		\bottomrule[1.5pt]
	\end{tabular}
	% \vspace{-3mm}
\end{table*}

\begin{table*}[t]
	\centering	
	\caption{Verifying permutation invariance on ViT and BERT in non-private and DPSGD modes. }
	\label{tab: permutation invariance}
	% \vspace{2mm}
	\begin{tabular}{ccccccccc}
		\toprule[1.5pt]
		Method		   & \multicolumn{2}{c}{non-private} & \multicolumn{2}{c}{shuffled non-private} & \multicolumn{2}{c}{DPSGD} & \multicolumn{2}{c}{shuffled with same $\sigma$} \\
		& test loss & test accuracy & test loss      & test accuracy  & test loss & test accuracy & test loss      & test accuracy   \\ \midrule
		ViT &      0.6362     &   89.68\%  & 0.6362     &   89.68\%   &      1.330     &   85.56\%  & 1.330     &   85.56\%            \\
		BERT           &      0.2302     &    92.31\%  &    0.2302     &    92.31\%  &      1.054     &    87.38\%  &    1.054     &    87.38\%  \\
		\bottomrule[1.5pt]
	\end{tabular}
\end{table*}

\section{Experiments}
% 这一章节我们介绍实验部分，主要包括两个大的方面，第一方面是对于PRV accountant的实验，第二方面是对于shuffled-DPSGD的实验。在第一部分，我们对于shuffle mechanism的PRV accountant进行数值试验，评价不同参数下的PRV accountant下的sigma的结果。在第二部分，我们对于shuffled-DPSGD进行大模型上的训练实验，评价不同参数下的shuffled-DPSGD的结果，并且对比DPSGD的结果。
We verify the theorems and evaluate the performance of the shuffled DPSGD in this section.  The experiments are designed to address the following key research questions: (1) Does the shuffling mechanism guarantee weights permutation invariance? (2) How does the F-W method perform in approximating the sum of log-normal distributions? (3) What is the relationship among $\sigma, \varepsilon, d$ in the shuffled setting? (4) How does the accuracy of large models improve by Shuffled DPSGD? (5) How does Shuffled DPSGD perform under auditing? (6) What is the overhead introduced by the shuffling mechanism? The answers to these questions are listed from Sec.~\ref{sec:perin} to Sec.~\ref{sec:overhead}.

\subsection{Experimental Setup}\label{sec:setup}

\textbf{Models, datasets and tasks.}
%这里我们介绍实验的任务和数据集，主要包括两个方面的任务，第一方面是CV的分类任务，第二方面是NLP的分类以及生成任务。对于CV的分类任务，我们通过训练ResNet-20、WideResNet 28-10、ViT这些模型分类CIFAR-10数据集，对于NLP的分类任务，我们采用的是训练BERT模型分类SST-2,QQP数据集，对于NLP的生成任务，我们采用的是基于E2E、Dart数据集训练GPT-2模型生成文本。我们将这些模型的参数以及数据集的介绍放在表XXX中。
We conduct comprehensive experiments on a range of models, datasets, and tasks in both CV and NLP fields. The first category of tasks is CV classification, the second being NLP classification and the third one is text generation. For CV classification, we train 
% three models --- ResNet18 \cite{he2006resnet}, WideResNet50-2 \cite{Zagoruyko2016wideresnet} and 
ViT \cite{Dosovitskiy2021vit} to classify the CIFAR-100 \cite{krizhevsky2009learning} dataset. For NLP classification tasks, we employ two models BERT\cite{Devlin2019bert} and RoBERTa on four datasets: stanford sentiment treebank (SST-2), quora question pairs (QQP),  Question Natural Language Inference (QNLI), and MultiNLI (MNLI) from the GLUE benchmark \cite{wang2019glue}. For the text generation task, we use GPT-2 and GPT-2-large \cite{radford2019language} to generate text given E2E \cite{Novikova2017e2e} and DART \cite{nan2021dart} datasets. ViT is pre-trained on ImageNet, whereas the rest models are adopted from the pre-trained ones in the Huggingface library. We summarize the configuration and training/fine-tuning hyperparameters in Table \ref{tab: training settings}. 

% Notice that we take the convolutional networks (ResNet-18 and WideResNet50-2) into account, although the convolutional layer is not permutation-invariant. Hence we fix all the convolutional layers and apply shuffling exclusively to the fully-connected layers which take a small proportion of the model weights. To ensure that the model satisfies permutation-invariant, we altered the final linear layer of ResNet18 and WideResNet50-2 models to two layers, connected by a ReLu activation layer, with hidden dimensions of 32 and 64, respectively. The rest of the models are Transformer-based and thus the entire model can be applied shuffling.

\textbf{Metrics.}
% 这里我们介绍实验的评价指标，对于CV和NLP的分类任务，我们采用的是accuracy作为评价指标，同时，为了验证我们提出的shuffle机制的Parameters Permutation Invariance,我们会引入Loss作为参考指标。对于NLP的生成任务，我们采用的是BLEU和Metor作为评价指标。
For the classification tasks of CV and NLP, we use the testing accuracy as the evaluation metric. To verify the weights permutation invariance of the shuffling mechanism, we additionally include loss as the metric. As to the text generation tasks, BLEU and Meteor are employed as the evaluation metrics. Except for the loss, a higher value indicates a better performance for all metrics.

\textbf{Baselines.}
% 这里我们引入几个DPSGD中较为前沿的工作作为baseline，主要包括Ghost clipping\cite{li2022large}、MixOpt\cite{bu2022scalable}, Opacus\cite{yousefpour2021opacus}、JAX-privacy\cite{de2022unlocking}. 其中Ghost clipping和MixOpt是针对per-sample gradients计算的进行了算法上的优化，Opacus是目前较为流行的基于pytorch的DP-SGD代码平台，代码库持续更新，并且同时集成了RDP，GDP和PRV三种不同的privacy accountant方式.
We introduce the state-of-the-art works on DPSGD as baselines, which are Ghost clipping \cite{li2022large} and MixOpt \cite{bu2023differentially}. Ghost clipping and MixOpt are algorithmically optimized for per-sample gradient calculation based on {\tt Opacus} \cite{yousefpour2021opacus} library. {\tt Opacus} is a popular pytorch-based DPSGD code platform, which is continuously updated, and offers three different privacy accountant methods: RDP, GDP and PRV. RDP accountant is adopted in our experiments as the accountant method for all baselines. In ghost clipping, the per-sample gradients is clipped without being instantiated, thus saving both running time and memory costs. MixOpt method is further improved from Book-Keep \cite{bu2022scalable} (A variant of ghostclipping) and provides a lighter running time overhead than Book-Keep.

\textbf{Implementation details.}
% 我们在一张3090显卡上进行实验的运行，实现的代码采用的是python，对于PRV accountant的实现，我们采用的是PRV accountant的官方代码库，对于其他的实现，我们采用的是基于MixOpt算法的代码库，同时，我们也在其基础上进行了shuffle机制的设计与部署。对于超参数的设置，我们将训练的epoch数设置为3，学习率设置为2e-5，batch size设置为1000，归一化/裁剪因子设置为0.1，delta设置为1/N。
We run the experiments on Nvidia GeForce RTX 3090 GPU except for training GPT-2 large on Nvidia A100 GPU. The implementation is Pytorch-based.  %We take the Fast-DP \cite{bu2023differentially} as the code base.
 Baselines are reproduced in their respective environments. In the hyperparameter setting of all NLP tasks, we set the clipping value to 0.1, and fix $\delta$ to $1/(2N)$ where $N$ is the number of training samples in each dataset. For CV tasks, we set the clipping value to 1, and fix $\delta = 5 \times 10^{-6}$. 
% A batch size of 1000 is selected for all except for the training on ResNet18 with a batch size of 8000. 
Other hyperparameters are listed in Table \ref{tab: training settings}. 

% The implementation of privacy accountant is facilitated by reusing the `DPSGDAccountant' from the PRV code library\footnote{\url{https://github.com/microsoft/prv_accountant}}. Notably, the official Opacus library also adopts this distribution for implementing PRV accountant.

\textbf{Shuffle implementation.}
% 其次对于shuffle的实际实现，我们这里采用的是sample随机的排列，作为矩阵的行列的index，而不是直接的进行矩阵乘法操作，这样的时间开销可以缩短非常多，同时，我们也可以保证shuffle的一致性。这里我们用简单的代码描述我们的shuffle操作：
For the implementation of the shuffling mechanism, we perform permutation by creating random row and column indices for each weight matrix, instead of directly performing matrix multiplications. This can greatly reduce the time overhead of which the results can be found at Sec.~\ref{sec:overhead}. The code detail is provided in Appendix \ref{appendix:shuffle code}.

\subsection{Verifying Permutation Invariance}\label{sec:perin}
% 首先我们验证理论上证明的shuffle机制的Parameters Permutation Invariance

%同时我们这里也提供了在固定随机种子下对模型参数进行shuffle的参数差值的对比实验，我们通过计算shuffle前后的参数的差值的l2norm，我们发现其差值在1063.02，因此我们可以得出结论，shuffle机制对于模型参数扰动是非常大的。但当我们采用逆向的shuffle操作，也就是对shuffle矩阵取转置再进行一次shuffle操作，我们发现其差值在0.0，因此我们可以得出结论，shuffle机制对于模型参数扰动是非常大的，但是逆向的shuffle操作可以将其还原。这一点也验证了我们的shuffle机制的Parameters Permutation Invariance。 
We first show the permutation invariance holds by conducting experiments to compare the difference in weights trained under shuffling and that without shuffling. We fix the random seed for a straightforward comparison. By calculating the $\ell_{2}$ norm of the weights difference pre- and post-shuffling on ViT, we observed a significant discrepancy of 1063.02, showing that the shuffling mechanism induces substantial perturbation in the model weights. 

Meanwhile, we verify permutation invariance by training ViT on CIFAR100 and Bert on SST-2 in both non-private mode and differentially-private mode. The same noise multiplier $\sigma$ (equivalent to $\varepsilon=1$ in normal DPSGD) is applied and the random seed is fixed. The training results with or without shuffling are listed in Table~\ref{tab: permutation invariance}, showing exactly the same test loss and accuracy with shuffling or not. It is evidenced that applying the proposed shuffling mechanism results in zero difference between the shuffled and original model, indicating that the model is trained equivalently under shuffling.

% But here's where it gets interesting: when we did the reverse shuffle, which means we basically unshuffled everything back to its original order, the difference dropped to zero. Yes, you read that right—zero difference. This means that even though shuffling can cause a big stir in the parameters, doing the reverse shuffle can perfectly undo all that mixing. It's like magic, proving that our shuffling method can make big changes but also reverse them completely. This experiment shows that our shuffling trick is not only cool but also reliable.

\begin{figure}[htbp]
	\centering
	\includegraphics[width=0.7\linewidth]{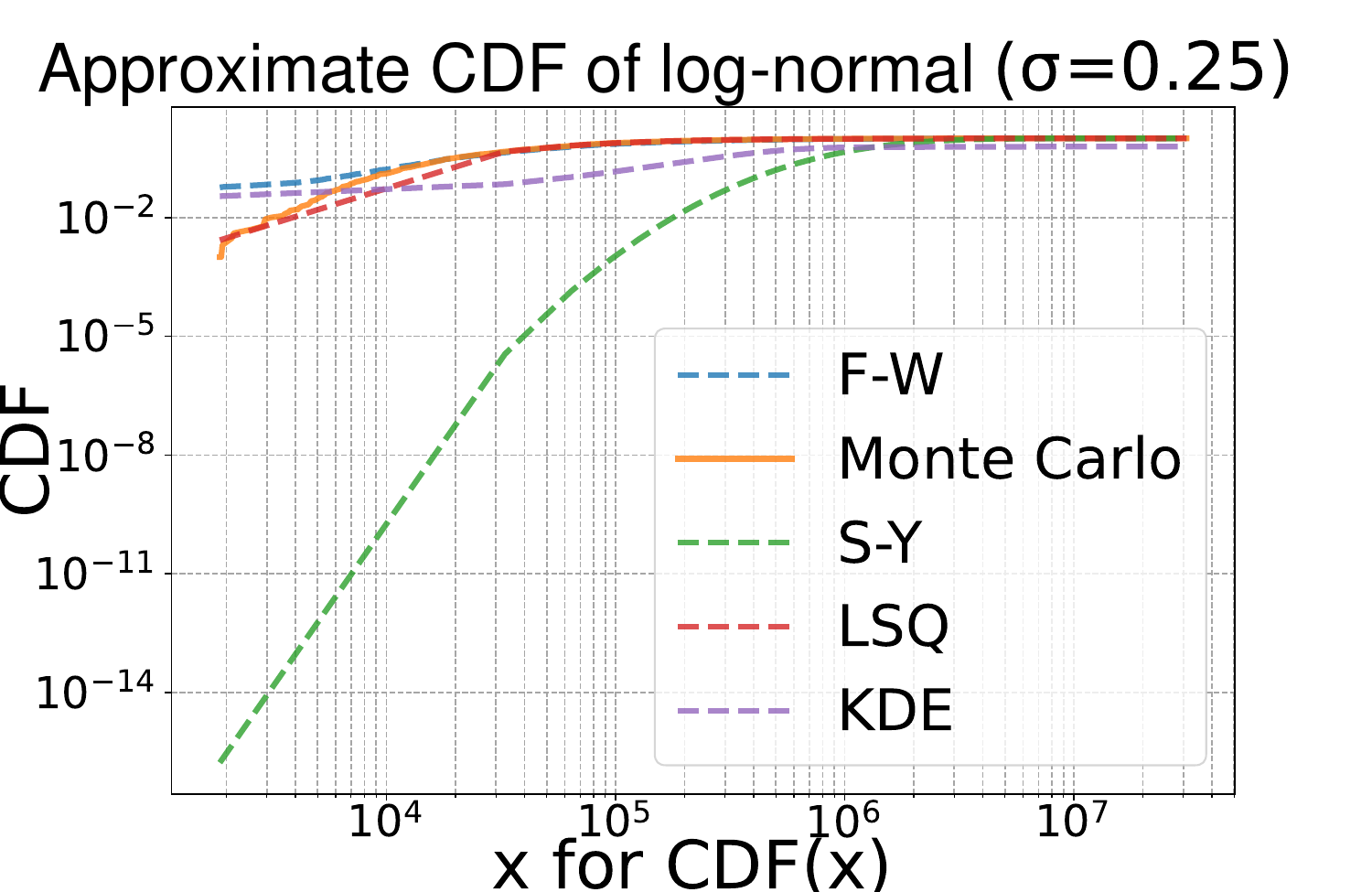}
	% \vspace{-3mm}
	\caption{The comparison of different approximation methods for the sum of log-normal distribution. We set $\sigma=0.25$ and dimension $d = 10^{8}$ to depict its CDF curve. }
	% \vspace{-3mm}
	\label{fig:lognormal_approximation}
\end{figure}

\begin{table*}[t]
	\centering
	\caption{Accuracy (\%) comparison of Shuffled DPSGD, Ghost clipping and MixOpt on CIFAR-100 classification tasks. The optimal performance is marked in bold.}
	\label{tab: result of cv}
	\tabcolsep=1.5pt
	% \vspace{-3mm}
	\begin{tabular}{ccccccccccc}
		\toprule[1.5pt]
		\multirow{2}{*}{Model} &             & \multicolumn{3}{c}{$\varepsilon=0.1$ }                & \multicolumn{3}{c}{$\varepsilon=0.5$ }                  & \multicolumn{3}{c}{$\varepsilon=1$  }                \\
		& Non-private & Ghost & MixOpt  & Shuffled-DPSGD   & Ghost & MixOpt  & Shuffled-DPSGD   & Ghost & MixOpt  & Shuffled-DPSGD   \\ \midrule
		% ResNet18  & 87.27     & 70.79       & 71.72 & \textbf{76.07} & 73.74       & 73.07 & \textbf{77.01} & 75.88       & 75.59 & \textbf{76.57} \\
		% WideRN50-2         & 88.37     & 79.33       & 79.53 & \textbf{82.68} & 81.88 & 81.79 & \textbf{83.39} & 83.35       & 83.97 & \textbf{84.99} \\
		ViT                    & 89.68     & 1.20       & 1.26 & \textbf{67.31} & 72.99 & 72.98 & \textbf{75.17} & 75.53       & 75.55 & \textbf{75.61}
		\\ \bottomrule[1.5pt]    
		\end{tabular}
		% \vspace{-3mm}
\end{table*}

\begin{table*}[t]
	\centering
	\caption{Accuracy (\%) comparison of Shuffled DPSGD, Ghost clipping and MixOpt on NLP classification tasks. The optimal performance is marked in bold.}
	\label{tab: result of nlp classification}
	\tabcolsep=1.5pt
		\begin{tabular}{cccccccccccc}
			\toprule[1.5pt]
			\multirow{2}{*}{Model}   & \multirow{2}{*}{Datasets} & \multirow{2}{*}{Non-private} & \multicolumn{3}{c}{$\varepsilon=0.1$} & \multicolumn{3}{c}{$\varepsilon=0.5$} & \multicolumn{3}{c}{$\varepsilon=1$} \\
									 &                           &                              & Ghost   & MixOpt   & Shuffled-DPSGD   & Ghost   & MixOpt  & Shuffled-DPSGD  & Ghost   & MixOpt  & Shuffled-DPSGD  \\ \midrule
									 \multirow{4}{*}{BERT}    & SST-2                     & 92.32                        & 69.27   & 64.56    & \textbf{87.5}    & 86.58   & 87.38    & \textbf{88.42}   & 88.19   & 86.93   & \textbf{88.30}  \\
									 & QQP                       & 90.49                        & 69.51   & 69.73    & \textbf{82.14}   & 80.85   & 80.51    & \textbf{82.25}   & \textbf{83.39 }  & 82.07   & {82.26}  \\
									 & MNLI                      & 83.27                        & 55.65   & 55.58    & \textbf{73.04}   & 70.33   & 70.12    & \textbf{73.05}   & \textbf{73.52}   & 73.25   & {73.07}  \\
									 & QNLI                      & 89.51                        & 51.29   & 59.69    & \textbf{81.49}   & 81.03   & 81.64    & \textbf{82.04}   & 81.77   & \textbf{81.82}   & {81.78}  \\ \midrule
			\multirow{4}{*}{RoBERTa} & SST-2                     & 94.95                        & 87.50   & 87.62    & \textbf{92.32}   & 90.82   & \textbf{92.08 }  & {91.63}   & 91.86   & 91.62   & \textbf{91.51}  \\
									 & QQP                       & 91.50                        & 68.19   & 68.28    & \textbf{83.59}   & 80.67   & 80.51    & \textbf{83.37}   & 83.39   & 83.20   & \textbf{83.35}  \\
									 & MNLI                      & 87.49                        & 55.11   & 53.03    & \textbf{79.29}   & 78.14   & 78.01    & \textbf{79. 21}  & 79.78   & 79.84   & \textbf{79.38}  \\
									 & QNLI                      & 93.56                        & 73.86   & 73.86    & \textbf{85.67}   & 85.22   & 85.92    & \textbf{85.94}   & 85.59   & 86.14   & \textbf{85.76} \\
									 \bottomrule[1.5pt]  
			\end{tabular}
\end{table*}

\subsection{Sum of Lognormal Approximations}
% 这里我们先对Privacy accountant 的数值结果进行实验和分析。首先，对shuffled-DPSGD算法的privacy accountant进行数值试验，评价不同参数下计算得到的sigma的大小，其次，比较在采用不同的近似方法，近似结果的差异性，以及对privacy accountant的影响。
% 最后，我们对于在定理3、4中计算满足差分隐私的条件中，所采用引理（F-W方法）对于sum of log-normal distribution 进行近似的方式进行近似误差的分析。并采用不同的近似方法进行比较，并分析近似方法的优劣程度。这里我们采用文章中所提供的一些常见的方法，例如，Q方法，
We empirically evaluate the approximation error of the F-W method \cite{Fenton1960} used in our mechanism, and compare it with different approximation methods including Monte-carlo method, S-Y method, LSQ method \cite{lognormal} and KDE method \cite{Askin22KDE}. Monte-carlo method calculates the sum of randomly sampled log-normal variables and repeats the procedure for 1000 times to depict the CDF curve. 
%S-Y方法采用迭代的方式，先计算两个log-normal distribution的和，然后再计算该和与第三个log-normal distribution的和，以此类推，直到计算出d个log-normal distribution的和。LSQ方法采用最小二乘法进行拟合，KDE方法采用高斯核密度估计的方法进行拟合。
S-Y method computes the distribution for the sum of log-normal iteratively. It first computes the distribution for the sum of two log-normal random variables, and then adds the third one up to compute the new distribution. The summation goes iteratively until all $d$ log-normal random variables are summed up. LSQ method leverages the least squares method to fit the sum of log-normal distributions, whereas KDE method adopts the Gaussian kernel density estimation method to fit the sum of log-normal distributions. 
%这里我们认为Monte-carlo method是最接近原始分布的方法，但是这种方法需要非常大的sample次数才能给出准确的估计，在本实验中，sample 1000次d=10^8的操作花费了62分钟的时间。完全不适合在计算privacy budget中使用。除了F-W方法，其余方法都是数值计算型方法，没有办法给出近似分布的CDF的解析解。其次LSQ和KDE方法都依赖于对sum进行sample才能完成近似，即使sample次数可以降低，但依然需要大量的时间。

Here we consider the Monte-carlo method to be the closest to the groundtruth distribution, but the method demands a very large number of samples to give an accurate estimate. For example, sampling 1000 points for $d=10^8$ took 62 minutes, which is hardly fit for DPSGD. Except for the F-W method, all other methods are numerical without any analytical form of the CDF. Furthermore, the LSQ and KDE methods both depend on sampling which is too heavy-weighted even if cutting down the number of samples.

%通过图中的比较，我们发现，这里大部分的近似方法都是比较接近的，
The results are shown in Fig.~\ref{fig:lognormal_approximation}. As we can see, the F-W method approximation of the sum of log-normal distribution is sufficiently close to the `ground truth' (Monte-carlo). S-Y method provides poor performance at a lower value. Although the LSQ method and KDE method have better performance, they require sampling similar to Monte-carlo and do not have a closed-form expression. Therefore, we conclude that the F-W method is the best choice for the approximation of the sum of log-normal in our design.

\begin{figure}[htbp]
	\centering
	\includegraphics[width=0.99\linewidth]{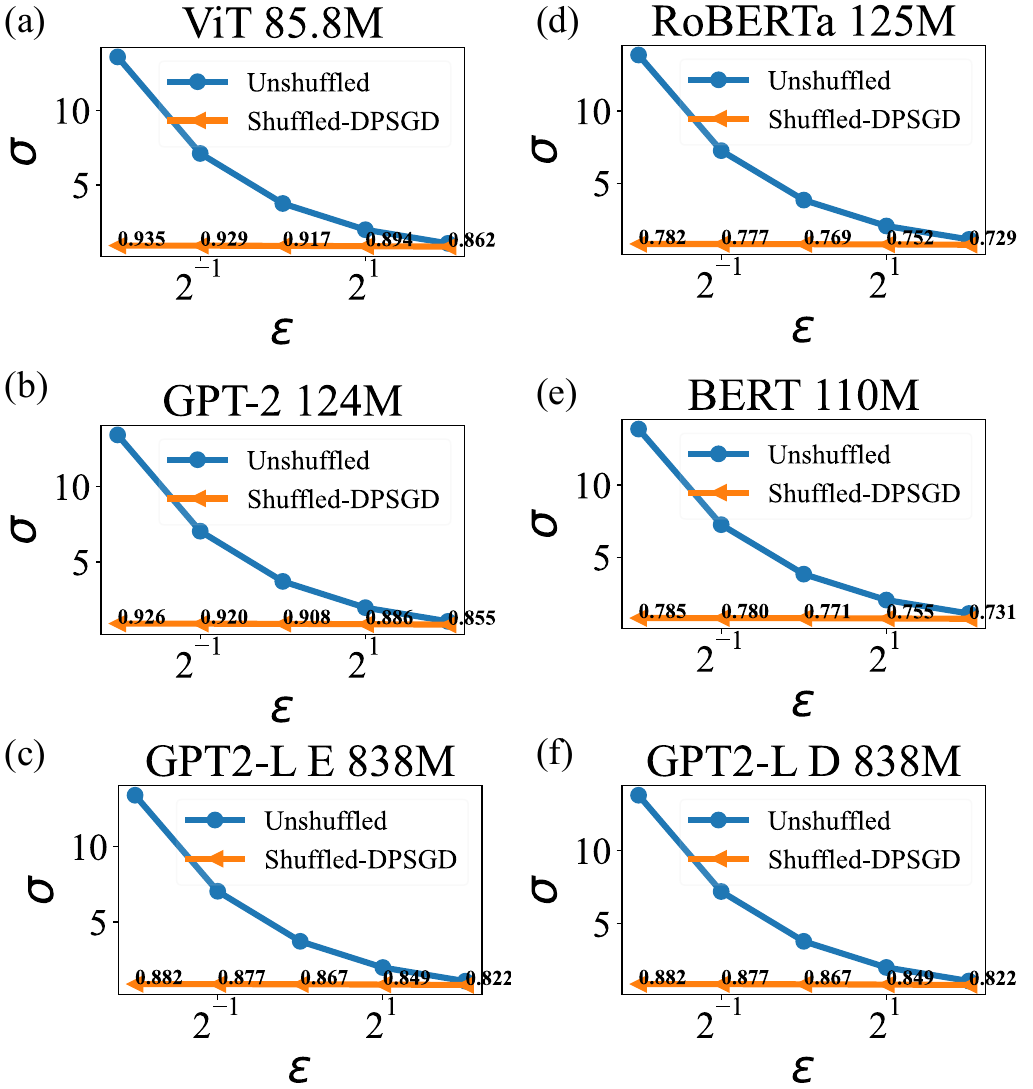}
	% \vspace{-3mm}
	\caption{Comparison of the relations between $\sigma$ and $\varepsilon$ in both unshuffled and shuffled DPSGD. For (a)-(e) we calculate different $\sigma$s with $\varepsilon \in \{0.25, 0.5, 1, 2, 4\}$. And "GPT2-L E" represent the GPT-2-large model trained with E2E dataset and "GPT2-L D" represent the GPT-2-large model trained with DART dataset. }
	% \vspace{-3mm}
	\label{fig:eps_and_sigma}
\end{figure}

\begin{table*}[t]
	\centering
	\caption{Comparison of Shuffled DPSGD, Ghost clipping and MixOpt on GPT-2 and GPT-2-large. The optimal performance is marked in bold.}
	\label{tab: NLP GPT-2}
	\tabcolsep=1.5pt
% 	\begin{tabular}{ccccccccc}
% 	\toprule[1.5pt]
% 	\multirow{2}{*}{Models} & \multirow{2}{*}{Datasets} & \multirow{2}{*}{Non-private} & \multicolumn{3}{c}{$\varepsilon=1$}                   & \multicolumn{3}{c}{$\varepsilon=3$}                  \\
% 	&  &  & Ghostclipping & MixOpt & Shuffled-DPSGD  & Ghostclipping & MixOpt & Shuffled-DPSGD   \\ \midrule
% 	\multirow{4}{*}{GPT-2} & DART(BLEU)    & 62.00         & 40.94         & 32.49  & \textbf{53.37}  & 41.15         & 41.83  & \textbf{55.03}  \\
% 	& DART (METEOR) & 0.474        & 0.356         & 0.341  & \textbf{0.454}  & 0.386         & 0.391  & \textbf{0.463}  \\
% 	& E2E (BLEU)    & 66.75       & 53.80         & 46.29  & \textbf{66.19}  & 57.48         & 56.78  & \textbf{66.69}  \\
% 	& E2E (METEOR)  & 0.4537      & 0.3328        & 0.3174 & \textbf{0.4368} & 0.346         & 0.3641 & \textbf{0.4385} \\
% 	\bottomrule[1.5pt]
% \end{tabular}
\begin{tabular}{cccccccccccc}
	\toprule[1.5pt]
	\multirow{2}{*}{Models}      & \multirow{2}{*}{Datasets} & \multirow{2}{*}{Non-private} & \multicolumn{3}{c}{$\varepsilon=0.1$} & \multicolumn{3}{c}{$\varepsilon=0.5$} & \multicolumn{3}{c}{$\varepsilon=1$} \\ 
								 &                           &                              & Ghost    & MixOpt  & Shuffled-DPSGD   & Ghost   & MixOpt  & Shuffled-DPSGD  & Ghost   & MixOpt  & Shuffled-DPSGD  \\ \midrule
								 \multirow{4}{*}{GPT-2}       & DART(BLEU)                & 62.00                        & 1.91     & 0.92    & \textbf{38.73}   & 23.71    & 33.51   & \textbf{40.56}   & 40.94   & 32.49   & \textbf{42.68}  \\
								 & DART (METEOR)             & 0.474                        & 0.005    & 0.008   & \textbf{0.333}   & 0.298    & 0.318   & \textbf{0.3612}  & 0.356   & 0.341   & \textbf{0.3859} \\
								 & E2E (BLEU)                & 66.75                        & 0.18     & 0.01    & \textbf{45.93}   & 35.80    & 44.79   & \textbf{49.77}   & 53.80   & 46.29   & \textbf{54.93}  \\
								 & E2E (METEOR)              & 0.4537                       & 0.0201   & 0.0108  & \textbf{0.3119}  & 0.2877   & 0.3091  & \textbf{0.3286}  & 0.3328  & 0.3174  & \textbf{0.3436} \\ \midrule
	\multirow{4}{*}{GPT-2-large} & DART(BLEU)                & 66.28                        & 5.74     & 4.73    & \textbf{53.45}   & 47.28    & 46.27   & \textbf{54.77}   & 52.36   & 52.63   & \textbf{53.97}  \\
								 & DART (METEOR)             & 0.521                        & 0.073    & 0.045   & \textbf{0.4570}  & 0.423    & 0.418   & \textbf{0.4791}  & 0.449   & 0.450   & \textbf{0.4936} \\
								 & E2E (BLEU)                & 68.80                        & 0.21     & 0.09    & \textbf{63.46}   & 30.97    & 33.82   & \textbf{64.51}   & 45.42   & 45.55   & \textbf{65.02}  \\
								 & E2E (METEOR)              & 0.4625                       & 0.0549   & 0.0487  & \textbf{0.3986}  & 0.2410   & 0.2590  & \textbf{0.4195}  & 0.3180  & 0.3231  & \textbf{0.4276} \\
								 \bottomrule[1.5pt]
	\end{tabular}
\end{table*}

% \begin{table*}[h]
% 	\caption{Comparison of Shuffled DPSGD, Ghost clipping and MixOpt on GPT-2-large.}
% 	\label{tab: NLP GPT-2-large}
% 	\tabcolsep=1.5pt
% 	\begin{tabular}{ccccccccc}
% 	\toprule[1.5pt]
% 	\multirow{2}{*}{Models} & \multirow{2}{*}{Datasets} & \multirow{2}{*}{Non-private} & \multicolumn{3}{c}{$\varepsilon=1$}                   & \multicolumn{3}{c}{$\varepsilon=3$}                  \\
% 	&  &  & Ghostclipping & MixOpt & Shuffled-DPSGD  & Ghostclipping & MixOpt & Shuffled-DPSGD   \\ \midrule
% 	\multirow{4}{*}{GPT-2-large} & DART(BLEU)                & 66.28                        & 52.36         & 52.63  & \textbf{62.73}  & 58.38         & 57.96  & \textbf{63.11}  \\
%                              & DART (METEOR)             & 0.521                        & 0.449         & 0.450  & \textbf{0.513}  & 0.478         & 0.481  & \textbf{0.515}  \\
%                              & E2E (BLEU)                & 68.80                        & 45.42         & 45.55  & \textbf{65.65}  & 55.47         & 60.24  & \textbf{67.54}  \\
%                              & E2E (METEOR)              & 0.4625                       & 0.3180        & 0.3231 & \textbf{0.4504} & 0.3596        & 0.3801 & \textbf{0.4534} \\
% 	\bottomrule[1.5pt]
% 	\end{tabular}
% 	% \vspace{-5mm}
% \end{table*}

\subsection{Relations of $\sigma, \varepsilon,$ \& $d$}\label{sec:exp_eps_sigma}

The theoretical relations of $\sigma, \varepsilon, d$ provided in Thm.~\ref{thm:dpshuffled} is too complicated to tell straightforwardly. Hence we numerically show how the three key parameters interact in the real settings.

% Furthermore, for the privacy accounting of shuffled-DPSGD, our implementation is guided by the results of Theorem 4. The outcomes of these comprehensive analyses are systematically illustrated in Figure XXX, providing a clear visual representation of our findings.

% 根据图中的结果，我们可以看到，在不同的模型维度下，shuffle机制对于DPSGD的privacy accountant的影响非常巨大，特别是在高维度的模型中，shuffle机制对于DPSGD的privacy accountant的影响更加明显。例如，对于ResNet-18和Wide ResNet 28-10模型，shuffle机制对于DPSGD的privacy accountant的影响相对小一些，而对于ViT、Bert、RoBERTa和GPT-2模型，shuffle机制对于DPSGD的privacy accountant的影响相对较大。这是因为，对于ResNet-18和Wide ResNet 28-10模型，shuffle机制只对全连接层进行了shuffle机制的操作，而对于ViT、Bert、RoBERTa和GPT-2模型，shuffle机制对于所有的参数进行了shuffle机制的操作。因此，我们可以得出结论，shuffle机制对于DPSGD的privacy accountant的影响是非常大的，特别是在高维度的模型中，shuffle机制对于DPSGD的privacy accountant的影响更加明显。其次我们也给出了在不同的eps的设定下的结果，我们可以看到，在传统的DPSGD的privacy accountant的计算中，eps的大小对于sigma的大小有着非常大的影响，而在shuffled-DPSGD中，eps的大小对于sigma的大小的影响相对较小，这是因为，shuffle机制大大降低了需要满足差分隐私的sigma值，因此，不同的eps的所计算出的sigma值较为接近。因此，我们可以看出在高隐私要求的情况（即eps<1）下，shuffled-DPSGD对于sigma的降低将会更加明显，因此模型的表现将会更加优异，这一点我们将在后面的实验中进行验证。
Due to space limitation, we only display results of six representative experiments, each with parameter size $d =$ 85.8M, 110M, 125M, 124M, 838M, respectively. The $\sigma-\varepsilon$ curves are depicted in Fig. \ref{fig:eps_and_sigma}. We observe that the shuffling mechanism has a significant impact on the privacy accountant of DPSGD across all model sizes. From the figure, the $\sigma$s almost flattened out across different $\varepsilon$s in shuffled mechanism, leaving other factors like $d$ to blame for its change in values. As we focus on the high privacy regime ($\varepsilon$ as low as $0.25$), we observe that, the $\sigma$ in the shuffled DPSGD is merely 5-7\% of that of the unshuffled.

It may be attributed to that the noise magnitude $\sigma$ at the order $O(1/\sqrt{\log d})$, which is very small at extremely large $d$s. We can tell that $d$ mainly appears along with $e^{\frac{1}{\sigma^2}}$ as $(d-1)e^{-\frac{d}{(d-1)\sigma^2}} ~\text{in}~ \sigma_{Z_i}, \zeta_{i}$, where $i=1,2$. Therefore, if the model's dimension reaches $10^{9}$, the $\sigma$ is dominated by $d$ but not by the $\varepsilon$.

For a more straightforward demonstration, we show the heatmap in Fig.~\ref{fig:heatmap} to present the relations of $\varepsilon, \sigma$ and dimensionality for shuffled DPSGD. The column of $d =1$ represents the $\sigma$ of an unshuffled mechanism. In the unshuffled case, $\sigma$ increases fast with the decrease of $\varepsilon$. However, as dimensionality $d$ increases, $\sigma$ drops drastically and the change in $\varepsilon$ barely affects $\sigma$. Therefore, in the shuffled case, $d$ becomes the major impact factor to the value of $\sigma$.

%这个现象主要取决于定理中的关于prv的计算，在公式中，我们可以看出，其中维度d主要和e同时出现，基本模块为d-1+e^{\frac{1}{\sigma^2}}。因此，当$\sigma=0.1$时，e^{\frac{1}{\sigma^2}} = e^{100} \approx 2\times 10^{34}。此时，即使模型的维度达到了10^{9}(billion),依然不值一提，因此，此时shuffled-DPSGD和unshuffled的\varepsilon结果非常接近。其次，当$\sigma=0.5$，e^{\frac{1}{\sigma^2}} = e^4 = 54.6。此时，模型的维度d一般为达到了10^{8}, 因此相对于\sigma的比重大得多，因此这时候的shuffled-DPSGD和unshuffled的\varepsilon结果差距就会非常大。因此，我们可以得出结论，当$\sigma\geq 0.3$的\varepsilon <1 取值较低时，shuffled-DPSGD表现将会更加优异，。

\begin{figure}[htbp]
	\centering
	\includegraphics[width=0.95\linewidth]{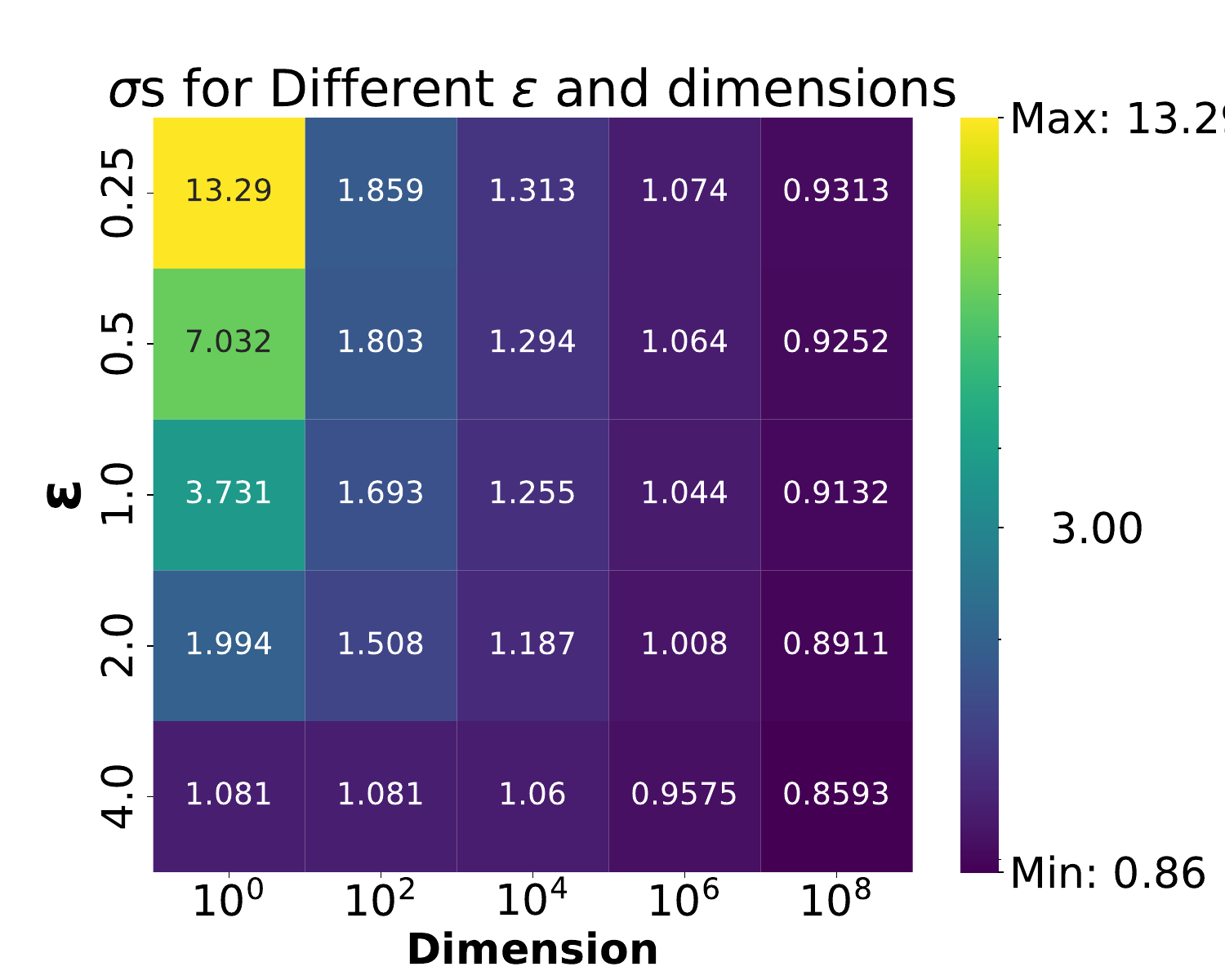}
	% \vspace{-5mm}
	\caption{The $\sigma$s from the unshuffled ($d=1$) and shuffled DPSGD under different $\varepsilon$s and $d$s.}
	\label{fig:heatmap}
\end{figure}

\subsection{Performance of Shuffled DPSGD}
% 接下来我们主要介绍shuffled-DPSGD以及baseline在大模型的训练中的结果。根据任务的不同，我们分为三个部分对结果进行介绍，首先是CV的分类任务，其次是NLP的分类任务，最后是NLP的生成任务。
We evaluate the real-world performance of shuffled DPSGD in comparison with the baselines on large-scale models. The results are reported according to the categories of the tasks.

\textbf{CV Classification.}
The setup of ViT follow the ViT checkpoint \texttt{google/vit} \texttt{-base-patch16-224} from Huggingface transformers library \cite{wolf2020transformers}, pretrained on ImageNet-21k and fine-tuned on ImageNet (also referred to as ILSVRC2012). Since this pre-trained model is very powerful, easily reaching over $96\%$ accuracy under all circumstances on CIFAR-10,  we hence omit the results but select a more difficult task --- 100-class classification on CIFAR-100 --- for performance test. The results in Table~\ref{tab: result of cv} show that our shuffled DPSGD consistently outperforms all baselines under various $\varepsilon$s, and the improvement is more significant at small $\varepsilon$ values.  Particularly, at $\varepsilon = 0.1$ where both Ghost clipping and MixOpt have performance close to random guesses, shuffled DPSGD enjoys a high accuracy well over $60\%$. 

However, such an improvement becomes marginal at large $\varepsilon$s. As we analyze, this is mostly due to the inexact approximation of F-W method at small $\sigma$s. In Thm.~\ref{thm:dpshuffled}, the variance of per normal distribution in the sum of log-normals is approximated as $\frac{c^2}{\sigma^2}$, which grows with small $\sigma$s. As stated in \cite{Zhang_2022}, the F-W approximation method is quite accurate when the variance of log-normal distributions is low but becomes less accurate at a high variance. Hence, at small $\sigma$s, or large $\varepsilon$s, the power of the approximation approach diminishes.

% 这里我们公布了在不同的eps下，我们的pretrain的结果，可以发现，我们的shuffled-DPSGD在不同的eps下，都能够取得比所有baseline更好的结果，特别是在eps较小的情况下，shuffled-DPSGD的结果更加优异。这是因为，我们的shuffled-DPSGD能够在保证相同隐私的情况下，采用更小的sigma进行加噪，这一点我们在前面的实验中已经进行了验证。

% We found that in ResNet18 and WideResNet50-2, the accuracy improvement by shuffled-DPSGD is about 2\% at $\varepsilon=1$ and about 1\% at $\varepsilon=3$. The improvement is not very significant mostly due to the relative easy CIFAR10 classification task. 

% Therefore its classification accuracies in the non-private, DPSGD, shuffled DPSGD settings are all over 96\%, with minimal gaps among them.

%在表格中，我们发现，在resnet18和WideResNet50-2中，shuffled-DPSGD在eps=1时精度提升在2%左右在eps=3时精度提升在1%左右，均与baseline有较为明显的提升。但ViT的pretrained模型效果较为优异，因此在non-private和DPSGD，shuffled-DPSGD方法中的分类acc都可以达到96%以上，所以结果差异并不明显。

\textbf{NLP Classification.} We finetune a raw BERT checkpoint \texttt{bert-} \texttt{base-uncased} and a RoBERTa checkpoint \texttt{roberta-base} in a differentially-private way for NLP classification, in which both pre-trained models are available in the Huggingface library. The results under all cases are presented in Table~\ref{tab: result of nlp classification}. The trend is similar to CV classification, as our method almost always achieves higher accuracy than baselines at the same privacy level. Notably, our accuracy is merely $1-4\%$ shy of the non-private case, showing great applicability of our method to privacy-preserving NLP tasks. In fact, in NLP tasks, quite different from the CV task, the accuracy of our method is only mildly affected by the $\varepsilon$ value. This reflects the results of Sec.~\ref{sec:exp_eps_sigma} where $\varepsilon$'s impact on $\sigma$ gets weakened in the high-dimensional case.

\textbf{Table-To-Text Generation.} 	
% 最后我们介绍shuffled-DPSGD在NLP Generation任务中的应用，我们采用了GPT-2模型在来训练E2E和DART数据集。
We fine-tune the pre-trained model checkpoint \texttt{gpt2} and \texttt{gpt2-large} from Huggingface in a differentially-private way. The  results are reported in Table \ref{tab: NLP GPT-2}. It can be found that the performance of the shuffled DPSGD is far superior to baselines, especially at $\varepsilon=0.1$. This is in line with the analysis that the accuracy performance of our method is less affected by the limited privacy budget thanks to shuffling. It is also clear that the performance gap between the shuffled DPSGD and the non-private method is smaller on GPT-2-large compared to GPT-2. In fact, at $\varepsilon=1$, the performance of our method on GPT-2-large is very close to that of the non-private case, demonstrating that our shuffling approach enjoys a unique advantage on large models, verifying the intuition that the shuffled DPSGD benefits from high-dimensional model weights.

%我们在表中呈现了在GPT-2上不同数据集和不同方法的训练结果。可以发现shuffled-DPSGD的表现特别优异，尤其是在eps=0.5的情况下，两个分数都大幅度高于baseline的实验结果。这符合前面对于更小的eps，sigma的降低更加明显的结论。其次，在dart数据集的结果中在eps=3的情况下，shuffled-DPSGD的表现非常接近于non-private的结果，这说明shuffled-DPSGD在大模型上可以训练得到非常好的模型。

\begin{table}[htbp]
	\centering
	\caption{Auditing results of shuffled DPSGD on ViT, CIFAR100.}
	\label{tab: aduting}
	\begin{tabular}{ccc}
	\toprule[1.5pt]
	% Privacy budget   & $\varepsilon = 0.5$ & $\varepsilon = 1$ \\ \midrule
		Theoretical $\varepsilon$  & 0.5          & 1.0    \\
		Audited $\varepsilon$ & 0.4771      & 0.7005
	     \\

	\bottomrule[1.5pt]
	\end{tabular}
\end{table}

\subsection{Privacy Aduting of Shuffled DPSGD}
% 目前shuffled-DPSGD的privacy accountant的计算是基于Mixture of Gaussian 来进行计算的，这里我们通过Aduting的方式来验证我们的shuffled-DPSGD的privacy accountant的计算的正确性。我们采用的是基于\cite{nasr2023tight}的Aduting方式，这是较为先进的可以在DPSGD中对与privacy进行Aduting的工作，可以在尽可能少的观测次数中，精准得的到privacy bound。
So far, our privacy accountant relies on some relaxations and approximations. We would like to see how our mechanism acts against real-world auditing. Hence we implement the auditing method for DPSGD in \cite{nasr2023tight}, which presents the most recent progress in DPSGD auditing as far as we know. We conduct auditing tests in the training of ViT on CIFAR-100 under the theoretical $\varepsilon= 0.5, 1$ by shuffled DPSGD. A total of 10,000 training steps is performed to get the records through `white-box access with gradient canaries.' The setup of `Dirac canary: All gradient values are zero except at a single index' in \cite{nasr2023tight} is reused in our test with a batch size 1000 and epoch number 200 to get sufficient observations.

We apply the auditing method to the shuffled DPSGD and obtain the auditing $\varepsilon$s as
\begin{equation}
	\varepsilon =\max \left\{\ln \left(\frac{1-{\alpha}-\delta}{{\beta}}\right), \ln \left(\frac{1-{\beta}-\delta}{{\alpha}}\right), 0\right\}
\end{equation}
where $\alpha$ and $\beta$ represent the type I error and type II error, respectively. 
We compare the empirical $\varepsilon$s with the theoretical values in Table~\ref{tab: aduting}. For $\varepsilon = 0.5$, we eliminate cases where $\alpha < 4 \times 10^{-4}$ and $\beta < 4 \times 10^{-4}$ to mitigate the impact of inaccurate measurement of the small probability events. The audited $\varepsilon$ is a reasonably close lower bound to the theoretical counterparts.

This indicates that our privacy accountant for shuffled DPSGD provides reasonably accurate estimations. However, there is still a discrepancy between the theoretical and empirical $\varepsilon$ values, which can be attributed to the approximations in the theoretical analysis and the fact that the audited privacy level is mostly a lower bound to the actual privacy level. Nevertheless, considering the limited observations ($10^4$), the auditing results is acceptable as the gap between the audited and the theoretical upper bound is not closed in \cite{nasr2023tight} either.

% It is found that the auditing results of DPSGD are close to tight, but the audited $\varepsilon$ is nowhere close to the theoretical one in shuffled DPSGD. In fact, the theoretical $\varepsilon$ in shuffled DPSGD sits between the unshuffled case and the audited $\varepsilon$, indicating shuffling indeed improves the privacy level but the privacy guarantee we provide in this paper is still a conservative one, leaving quite a margin to the auditing result. This is attributed to that, we treat the total number of possible permutations as the same as the model dimension $d$ in the privacy guarantee derivation, but the real total number is far beyond $d$ --- it could reach as high as $d!$ if all values are different in the vector --- which is highly likely due to the floating-point number arithmetics. The privacy guarantee in our paper is merely provided for the unlikely worst case, and thus is an overly conservative one, despite that it already well surpasses baselines.

% 我们发现对于DPSGD的Aduting结果已经接近tight，其理论和实际检测出的$\varepsilon$较为接近。关于shuffled-DPSGD的结论，可以看出，shuffled机制带来的扰动较高，使得Aduting得到的eps有所下降。但shuffled-DPSGD的检测出的$\varepsilon$并没有接近tight的理论上界。我们认为这是我们在理论过程中采用了second-worst case的原因，即f(x'）= e_1, 使得shuffle扰动的结果数从d！变成d。这个结论会使得我们的理论varepsilon更加保守，因此检测出来的varepsilon会比理论的更小。

\begin{table}[htbp]
	\centering
	\caption{Time costs of different shuffling implementations.}
	\label{tab: shuffle time}
	\begin{tabular}{cc}
		\toprule[1.5pt]
		method                    & Times(ms) \\ \midrule
		Index shuffle             & 0.9799       \\
		Matrix multiplication   & 1316.6   \\
		Sparse matrix multiplication & 322.62   \\\bottomrule[1.5pt]
	\end{tabular}
\end{table}

\begin{table}[htbp]
	\centering
	\caption{Average epoch time of running shuffled DPSGD.}
	\label{tab: overhead}
	\begin{tabular}{ccc}
		\toprule[1.5pt]
		Model   & Unshuffled & Shuffled \\ \midrule
		% ResNet18, CIFAR  & 15s       & 22s     \\
		% WideResNet50-2, CIFAR  &     1min12s      & 1min21s    \\
		ViT, CIFAR  &     2min25s      & 2min42s    \\
		BERT, QNLI & 28min28s & 28min53s \\
		RoBERTa, QNLI & 29min45s & 30min11s \\
		GPT-2, E2E & 30min1s & 30min13s \\
		GPT-2-large, E2E (on A100) & 43min25s & 44min29s \\
		\bottomrule[1.5pt]
	\end{tabular}		
\end{table}

\subsection{Overhead of Shuffled DPSGD} \label{sec:overhead}
We compare the overhead of shuffling under different implementations: index operations (our approach), matrix multiplication (\texttt{torch.mm}), and sparse matrix multiplication (\texttt{torch.sparse.mm}) in randomly shuffling a $10,000 \times 10,000$ matrix. We found that index operations greatly save the running time, which is about 0.1\% of that by matrix multiplication and 0.3\% of that by sparse matrix multiplication. The results are in Table \ref{tab: shuffle time}.

% % 很明显，shuffled-DPSGD是需要额外的计算开销的，这里我们对于shuffled-DPSGD的计算开销进行了分析，其主要的开销源自于在训练过程中，对于参数的shuffle操作，这里我们对于不同的模型，不同的参数进行了shuffle操作，然后对比了shuffle操作的时间开销。尽管我们这里是对所有的参数进行了shuffle，但这里我们实际的shuffle操作的参数是非常快的。尽管我们在理论中，我们采用置换矩阵进行了shuffle操作，但是在实际的实现中，我们采用的是随机扰动index，并且将其作为参数矩阵的行列index。这样的操作，可以大大降低shuffle操作的时间开销。这里我们将shuffle操作的时间开销放在表XXX中。
The overall computational overhead of shuffled DPSGD is measured by the average epoch running time as reported in Table~\ref{tab: overhead}. All experiments are conducted on a single RTX 3090 GPU except for GPT-2 large. We can see that the additional overhead of shuffling is under 30 seconds except for GPT-2-large, which is 64 seconds, mostly due to its billion-level parameters and the enormous permutation space. Nevertheless, the shuffling overhead takes only 2.3\% of the overall training time in shuffled DPSGD, which is almost negligible in training.

%% file: arxiv/conclusion.tex
\section{Conclusion}
We developed the theoretical condition for $(\epsilon, \delta)$-DP to hold under shuffled Gaussian mechanism. Based on the theory, we proposed shuffled DPSGD, a method that enhances the accuracy of DPSGD by incorporating a shuffling mechanism to reduce the additive noise, thus boosting the overall model performance. Our shuffling mechanism, designed to maintain weights permutation invariance, is adaptable to a variety of up-to-date neural networks. Experimental results confirmed our findings, and showed that shuffled DPSGD significantly improves performance over baselines on large models, highlighting its lightweighted implementation and running time overhead.

%% file: arxiv/appendix.tex
\appendix
\section{Proofs of Theorems }
\subsection{Proof of Theorem 1}
\label{proof:theorem1}

\begin{proof}
   
    As we have
    \begin{align}
        P(W)=\{P_{1}^{\top} W_t, ~b_t P_1,~W_{t+1} P_1 ,~ ~b_{t+1}\},
    % \begin{cases}
    %     & P^{\top} W_t, ~b_t P,\\
    %     & W_{t+1} P,~ ~b_{t+1}.\\
    % \end{cases}
    % \label{eq: linear permutation}
    \end{align}
    
    %然后我们将\ref{eq: linear permutation}带入上述公式中，即可得到
    we can get the following equations by substituting the above into Eq.~\ref{eq:lineareq1} and \ref{eq:lineareq2}:
        \begin{align}
        & h(x_{t} W_{t}^{\top} P_1 + b_{t} P_1) = x_{t+1}P_1,\\
        & h(x_{t+1} P_1 P_1^{\top} W_{t+1}^{\top} + b_{t+1}) = x_{t+2}.
        \end{align}
    %因此我们可以得到$x_{t+2}^{\prime} = x_{t+2}$，即我们可以得到shuffle不变得结果。
    Therefore, the forwarding result is still $x_{t+2}$, satisfying forward invariance.

    For the backward propagation, we can get the gradients as:
    \begin{align}
        & \frac{\partial x_{t+2}}{\partial W_{t+1} P_1} = \frac{\partial x_{t+2}}{\partial W_{t+1}} \frac{\partial W_{t+1}}{\partial W_{t+1} P_1}  = \frac{\partial x_{t+2}}{\partial W_{t+1}} P_1,\\
        & \frac{\partial x_{t+2}}{\partial b_{t+1}} = \frac{\partial x_{t+2}}{\partial b_{t+1}}\\
        & \frac{\partial x_{t+2}}{\partial P_{1}^{\top} W_{t}} = \frac{\partial x_{t+2}}{\partial W_{t} } \frac{\partial W_{t}}{\partial P_{1}^{\top} W_{t}} = P_1^{\top} \frac{\partial x_{t+2}}{\partial W_{t} } ,\\
        & \frac{\partial x_{t+2}}{\partial b_{t} P_1} = \frac{\partial x_{t+2}}{\partial b_{t}} \frac{\partial b_{t} }{\partial  b_{t} P_1} = \frac{\partial x_{t+2}}{\partial b_{t}} P_1.        
    \end{align}
    Therefore, the gradients have the same permutation order with the weights, and thus the results of the backward propagation are not changed.
\end{proof}

\subsection{Proof of Theorem 2}
\label{proof:theorem2}
\begin{proof}
    
    %根据下面的公式，我们可以对上面QK的内积部分进行化简计算，得到
    For the permuted attention block, we simplify the inner product of QK to get:
    \begin{equation}
        W_i^{Q} P_{1i} P_{1i}^{\top} W_i^{K \top} = W_i^{Q} W_i^{K \top}.
    \end{equation}
    
    %因此我们可以得到
    Hence, we get:
  
    \begin{align}
        \operatorname{softmax}\left(\frac{X_t W_i^{Q} W_i^{K \top}X_t^{\top}}{\sqrt{d_k}}\right) X_t W_i^{V} P_2 = A_i P_2.
    \end{align}
    
    %然后我们通过attention的定义，将A_i组合在一起，并且与W^{O\prime}相乘，得到    
    Combining $A_i$s together according to the definition of attention, and multiplying the result with $W^O$, we have:
   
    \begin{align}
        \left[ A_1 P_2, \ldots, A_h P_2 \right] P_3 W^{O}
        = \left[ A_1, \ldots, A_h \right] W^{O} = X_{t+1}.
    \end{align}
    Therefore, we can get the result of shuffling does not alter from the original one in the forwarding pass. 
    
    For the backward propagation, we get that the gradients as:
    \begin{align}
        & \frac{\partial X_{t+1}}{\partial W_i^{Q} P_{1i}} = \frac{\partial X_{t+1}}{\partial W_i^{Q}} \frac{\partial W_i^{Q}}{\partial W_i^{Q} P_1}  = \frac{\partial X_{t+1}}{\partial W_i^{Q}} P_{1i},\\
        & \frac{\partial X_{t+1}}{\partial W_i^{K} P_{1i}} = \frac{\partial X_{t+1}}{\partial W_i^{K}} \frac{\partial W_i^{K}}{\partial W_i^{K} P_{1i}}  = \frac{\partial X_{t+1}}{\partial W_i^{K}} P_{1i},\\
        & \frac{\partial X_{t+1}}{\partial W_i^{V} P_2} = \frac{\partial X_{t+1}}{\partial W_i^{V}} \frac{\partial W_i^{V}}{\partial W_i^{V} P_2}  = \frac{\partial X_{t+1}}{\partial W_i^{V}} P_2,\\
        & \frac{\partial X_{t+1}}{\partial P_3 W^{O} } = \frac{\partial X_{t+1}}{\partial W^{O}} \frac{\partial W^{O}}{\partial P_3 W^{O} }  = P_3 \frac{\partial X_{t+1}}{\partial W^{O}} .
    \end{align}
    Therefore, the gradients share the same permutation order with the weights, demonstrating backward invariance in the weight update.
    
\end{proof}

\subsection{Proof of Theorem \ref{thm:dpshuffled}}\label{appendix:theorem3}
    \begin{lemma}\label{lemma:inequality}
        Assuming $a_1, a_2, \ldots, a_k$ and $b_1, b_2, \ldots, b_k$ are positive real numbers, we have
        \begin{align}
            \sum_{i=1}^{k} a_i\sum_{i=1}^{k} b_i \geq \sum_{i=1}^{k} a_i b_i.
        \end{align}
    \end{lemma}
    \begin{proof}
        Obviously,
        \begin{align}
            &\sum_{i=1}^{k} a_i\sum_{i=1}^{k} b_i = \sum_{i=1}^{k} a_i b_i + \sum_{i\neq j} a_i b_j >  \sum_{i=1}^{k} a_i b_i.
        \end{align}
    \end{proof}
Substituting Eq.~\eqref{eq:22} into the above inequality, we obtain
    \begin{align}
        & \frac{\sum_{i=1}^{k} e^{<y, f^{(i)}>/\sigma^2}}{\sum_{i=1}^{k} e^{( <y, f^{(i)}> + <y, g^{(i)}> )/\sigma^2}} \geq \frac{1}{ \sum_{i=1}^{k} e^{(<y, g^{(i)}> )/\sigma^2}}.
    \end{align}
    Therefore, a sufficient condition for $O^{*}$ is to have
    \begin{align}
        & \frac{1}{ e^{-(\|g\|^2_2+ 2<f, g>)/ (2\sigma^2) }\sum_{i=1}^{k} e^{(<y, g^{(i)}> )/\sigma^2}} \geq e^{\epsilon} \\
        \Leftrightarrow &  e^{ - \epsilon + (\|g\|^2_2+ 2<f, g>)/ (2\sigma^2)} \geq \sum_{i=1}^{k} e^{(<y, g^{(i)}> )/\sigma^2}.
    \end{align}

\subsection{Proof of Lemma~\ref{lemma:increasing}}\label{appendix:lemma3}
\begin{proof}
   In $h(\sigma_{Z_1}, \sigma_{Z_2}) = \Phi \left(\frac{\sigma_{Z_1}}{2} + \frac{\zeta_1 - \epsilon}{\sigma_{Z_1}} \right)- e^{\epsilon}\Phi \left(\frac{\sigma_{Z_2}}{2} + \frac{\zeta_2 - \epsilon}{\sigma_{Z_2}} \right)$, for the definitions of $\zeta_1, \zeta_2, \sigma_{Z_1}, \sigma_{Z_2}$, we can easily derive
    \begin{align}
        & \zeta_1 \le 0, ~~\zeta_2 \le \zeta_1 - \frac{\|g\|_2^2}{\sigma^2},\\
            &  \log \left[\frac{ 1 }{k} (e^{\frac{\|g\|_2^2}{\sigma^2}} -1) + 1 \right]< \sigma_{Z_i}^2 < \frac{\|g\|^2_2}{\sigma^2},~~i=1,2.\label{eq:z1_z2_inequality}
    \end{align}
By analysis, we found $h(\cdot)$ is mostly influenced by the term $
        \sum_{j=1}^{k}e^{2\mu_{ij}},  \sum_{j=1}^{k}e^{\mu_{ij}},
        \sum_{j=1}^{k}e^{2\mu_{ij}^{\prime}},$ and $  \sum_{j=1}^{k}e^{\mu_{ij}^{\prime}}$.
Letting
    \begin{align}
        \bm{\mu} &= \left[ e^{\mu_{i1}},  e^{\mu_{ik}}, \ldots, e^{\mu_{ik}}  \right] \in \mathbb{R}^{k}, \\
        A & = \text{diag}\{ e^{\frac{<g^{(i)}, g^{(1)}>}{\sigma^2}}, e^{\frac{<g^{(i)}, g^{(2)}>}{\sigma^2}}, \ldots,  e^{\frac{<g^{(i)}, g^{(k)}>}{\sigma^2}}\} \in \mathbb{R}^{k \times k},
    \end{align}
    one can rewrite the terms as
    \begin{align}
        \sum_{j=1}^{k}e^{2\mu_{ij}} = \|\bm{\mu}\|^2_2, ~~ \sum_{j=1}^{k}e^{\mu_{ij}} = \|\bm{\mu}\|_1 \\
        \sum_{j=1}^{k}e^{2\mu_{ij}^{\prime}} = \|A\bm{\mu}\|^2_2, ~~ \sum_{j=1}^{k}e^{\mu_{ij}^{\prime}} = \|A\bm{\mu}\|_1.
    \end{align}
    Substituting the above shorthand into the variables, we have
    \begin{align}
        \sigma_{Z_1}^2 &= \log \left[\frac{\|\bm{\mu}\|^2_2}{ \|\bm{\mu}\|_1^2}(e^{\frac{\|g\|^2}{\sigma^2}} -1) + 1\right], \\
        \sigma_{Z_2}^2 &= \log \left[\frac{\|A\bm{\mu}\|^2_2}{ \|A\bm{\mu}\|_1^2}(e^{\frac{\|g\|^2}{\sigma^2}} -1) + 1\right], \\
        \zeta_1 &= -\log(\|\bm{\mu}\|_1) + \frac{<f, g>}{\sigma^2}, \\
        \zeta_2 &= -\log(\|A\bm{\mu}\|_1) + \frac{<f, g>}{\sigma^2}.
    \end{align}
    Since $\sigma_{Z_1}^2$ and $\sigma_{Z_2}^2$ only appear in $h(\cdot)$ as square root magnitudes, their influence is relatively small compared to the changes in $\bm{\mu}$ and $A$. Therefore, we mainly consider the impact of $\zeta_1$ and $\zeta_2$ on $h$. It is not difficult to see that the larger the difference $\zeta_1 - \zeta_2$, the greater the value of the function $h$. Additionally, the larger $\zeta_1$ is, the greater the value of the function $h$. Moreover, since $\|A\bm{\mu}\|_1 \le \|A\|_1 \|\bm{\mu}\|_1$ holds by norm inequality, we have:
    \begin{align}
        \zeta_2 \geq -\log\|A\|_1 - \log(\|\bm{\mu}\|_1) + \frac{<f, g>}{\sigma^2} = \zeta_1 -\log\|A\|_1.
    \end{align}

 By the definition of $A$, we have $\|A\|_1 = e^{\frac{\|g\|_2^2}{\sigma^2}}$, since $e^{\frac{\|g\|_2^2}{\sigma^2}}$ is the maximal element in the diagonal of $A$. The largest gap between $\zeta_1$ and $\zeta_2$ is achieved at $\|A\bm{\mu}\|_1 = \|A\|_1 \|\bm{\mu}\|_1$ which holds under two conditions: 1) $A = e^{\frac{\|g\|_2^2}{\sigma^2}}E$, where $E$ is the identity matrix; 2) $A\bm{\mu} = e^{\frac{\|g\|_2^2}{\sigma^2}}\bm{\mu}$, where $\bm{\mu}$ is an eigenvector of the matrix $A$ corresponding to the eigenvalue $e^{\frac{\|g\|_2^2}{\sigma^2}}$. Below, we will discuss each case separately and determine the conditions to achieve the maximum value of $h$.

The first case implies that the inner products of $g^{(i)}$ and $g^{(j)}$ are all equal to $\|g\|_2^2$, which is almost impossible for $k > 1$. To exclude the case where every $g^{(i)}$ is close to the vector $ [1,1, \ldots, 1] $, we ensure that in the high-dimensional space, the probability that the angle between $g^{(i)}$ and $[1, 1, \ldots, 1]$ being less than $\pi/4$ is almost zero. Therefore, we will not consider this case here.

For the second case, it is clear that the eigenvectors of the diagonal matrix $A$ are standard base vectors $e_i$ which has a 1 in the $i$-th position and 0 elsewhere. Therefore, the vector $\bm{\mu}$ is close to the direction of one of these base vectors. According to the observation that the larger $\zeta_1$ is, the larger the value of the function $h$ will be, and thus we need to maximize $\zeta_1$.

    Since
    \begin{align}
        \zeta_1 &= -\log(e^{\frac{<f, g>}{\sigma^2}} + \sum_{j\neq i} e^{\frac{<f^{(i)}, g^{(j)}>}{\sigma^2}}) + \frac{<f, g>}{\sigma^2} \\
        % & = -\left[\log(e^{\frac{<f, g>}{\sigma^2}} + \log(1 + \sum_{j\neq i} e^{\frac{<f^{(i)}, g^{(j)}> - <f, g>}{\sigma^2}})\right]+ \frac{<f, g>}{\sigma^2} \\
        & = - \log( \sum_{j =1}^{k} e^{\frac{<f^{(i)}, g^{(j)}> - <f, g>}{\sigma^2}}),
    \end{align}
	the larger $<f, g>$ is, and the smaller $<f^{(i)}, g^{(j)}>$ ($j \neq i$) is, the larger $\zeta_1$ becomes. Moreover, as $<f^{(i)}, g^{(j)}>$ is lower bounded by
	  \begin{align}
		e^{\mu_{ij}} = e^{\frac{<f^{(i)}, g^{(j)}>}{\sigma^2}} \geq e^{-\frac{\|f\|_2\|g\|_2}{\sigma^2}}, \label{mu_lower_bound}
	\end{align}
	 the smaller the value of $k$, the larger $\zeta_1$ is, with the minimum value of $k$ being $d$. Therefore, given the constraints $\|g\|_2 \le c$ and $\|f\|_2 \le c'$, we reformulate the problem to maximize the following optimization problem:
    \begin{align}
      \text{maximize}  ~\sum_{j = 1}^{k} (<f, g> - <f^{(i)}, g^{(j)}>).
    \end{align}
    Note that $<f, g>$ is maximized when $f$ and $g$ are of the same direction. The original problem can be transformed into:
    \begin{align}
       \text{maximize} ~  k(<f, g> - <f^{(i)}, \frac{1}{k}\sum_{j = 1}^{k}g^{(j)}>)
    \end{align}
    Since $f$ and $g$ are in the same direction, we must have $<f^{(i)}, \frac{1}{k}\sum_{j=1}^{k}g^{(j)}> \geq 0$. When $<f^{(i)}, \frac{1}{k}\sum_{j=1}^{k}g^{(j)}> = 0$, $\zeta_1$ reaches its maximum value. In this case, since $\frac{1}{k}\sum_{j=1}^{k}g^{(j)}$ is parallel to $[1,1,\ldots,1]$, and $f$ and $g$ are in the same direction, then $g \bot [1,1,\ldots,1]$. Meanwhile, since $k=d$ when $\zeta_1$ is maximized, $g$ cannot be in any other form other than having one dimension different with the rest being the same. Combining the two criteria, we could solve $g$ and $f$ as
    \begin{align}\label{appendix: eq_worst_case}
        f = \sqrt{\frac{1}{d(d-1)}}c^{\prime}[(d-1), -1, \ldots, -1] \\
        g = \sqrt{\frac{1}{d(d-1)}}c[(d-1), -1, \ldots, -1].
    \end{align}
    Under these conditions, we have
    \begin{equation}\label{appendix: eq_zeta_sigma}
        \begin{split}
        \zeta_1 & = -\log(1 + (d-1)e^{-\frac{d c c^{\prime}}{(d-1)\sigma^2}}), \\
        \zeta_2 & = -\log(1 + (d-1)e^{-\frac{d c (c+c^{\prime})}{(d-1)\sigma^2}}) - \frac{c^2}{\sigma^2}, \\
        \sigma_{Z_1}^2 &= \log \left[\frac{ 1 + (d-1)e^{-\frac{2 d c c^{\prime}}{(d-1)\sigma^2}} }{(1 + (d-1)e^{-\frac{dc c^{\prime}}{(d-1)\sigma^2}})^2} (e^{\frac{\|g\|_2^2}{\sigma^2}} -1) + 1 \right], \\
        \sigma_{Z_2}^2 &= \log \left[\frac{ 1 + (d-1)e^{-\frac{2 d c (c+c^{\prime})}{(d-1)\sigma^2}} }{(1 + (d-1)e^{-\frac{dc (c + c^{\prime})}{(d-1)\sigma^2}})^2} (e^{\frac{\|g\|_2^2}{\sigma^2}} -1) + 1 \right],              
        \end{split}
    \end{equation}
    which complete the proof.
\end{proof}

\section{Appendix of Experiments}
\subsection{Python Code for Index Shuffling}
\label{appendix:shuffle code}
\begin{lstlisting}[language=Python]
def rearrange(matrix, row, column):
 if row:
   indices = torch.randperm(matrix.shape[0])
   matrix = matrix[indices]
 if column:
   indices = torch.randperm(matrix.shape[1])
   matrix = matrix[:, indices]
 return matrix
\end{lstlisting}
%这里我们对采用index进行shuffle和采用矩阵乘法以及稀疏矩阵乘法的方法进行了比较，我们发现，采用index进行shuffle的方法，时间开销可以缩短非常多，大约是矩阵乘法的千分之一，是稀疏矩阵乘法的三百分之一。结果如下：

%% file: main.bbl
% Generated by IEEEtranS.bst, version: 1.12 (2007/01/11)
\begin{thebibliography}{10}
\providecommand{\url}[1]{#1}
\csname url@samestyle\endcsname
\providecommand{\newblock}{\relax}
\providecommand{\bibinfo}[2]{#2}
\providecommand{\BIBentrySTDinterwordspacing}{\spaceskip=0pt\relax}
\providecommand{\BIBentryALTinterwordstretchfactor}{4}
\providecommand{\BIBentryALTinterwordspacing}{\spaceskip=\fontdimen2\font plus
\BIBentryALTinterwordstretchfactor\fontdimen3\font minus
  \fontdimen4\font\relax}
\providecommand{\BIBforeignlanguage}[2]{{%
\expandafter\ifx\csname l@#1\endcsname\relax
\typeout{** WARNING: IEEEtranS.bst: No hyphenation pattern has been}%
\typeout{** loaded for the language `#1'. Using the pattern for}%
\typeout{** the default language instead.}%
\else
\language=\csname l@#1\endcsname
\fi
#2}}
\providecommand{\BIBdecl}{\relax}
\BIBdecl

\bibitem{abadi2016deep}
M.~Abadi, A.~Chu, I.~Goodfellow, H.~B. McMahan, I.~Mironov, K.~Talwar, and
  L.~Zhang, ``{Deep Learning with Differential Privacy},'' in \emph{Proceedings
  of the 2016 ACM SIGSAC Conference on Computer and Communications Security
  (CCS)}.\hskip 1em plus 0.5em minus 0.4em\relax ACM, 2016, pp. 308--318.

\bibitem{Anil2021LargeScaleDP}
R.~Anil, B.~Ghazi, V.~Gupta, R.~Kumar, and P.~Manurangsi, ``Large-scale
  differentially private bert,'' in \emph{Conference on Empirical Methods in
  Natural Language Processing}, 2021.

\bibitem{Askin22KDE}
\BIBentryALTinterwordspacing
{\"{O}}.~Askin, T.~Kutta, and H.~Dette, ``Statistical quantification of
  differential privacy: {A} local approach,'' in \emph{43rd {IEEE} Symposium on
  Security and Privacy, {SP} 2022, San Francisco, CA, USA, May 22-26,
  2022}.\hskip 1em plus 0.5em minus 0.4em\relax {IEEE}, 2022, pp. 402--421.
  [Online]. Available: \url{https://doi.org/10.1109/SP46214.2022.9833689}
\BIBentrySTDinterwordspacing

\bibitem{balle2019privacy}
B.~Balle, J.~Bell, A.~Gasc{\'o}n, and K.~Nissim, ``The privacy blanket of the
  shuffle model,'' in \emph{Advances in Cryptology -- CRYPTO 2019},
  A.~Boldyreva and D.~Micciancio, Eds.\hskip 1em plus 0.5em minus 0.4em\relax
  Cham: Springer International Publishing, 2019, pp. 638--667.

\bibitem{bittau2017prochlo}
\BIBentryALTinterwordspacing
A.~Bittau, U.~Erlingsson, P.~Maniatis, I.~Mironov, A.~Raghunathan, D.~Lie,
  M.~Rudominer, U.~Kode, J.~Tinnes, and B.~Seefeld, ``Prochlo: Strong privacy
  for analytics in the crowd,'' in \emph{Proceedings of the 26th Symposium on
  Operating Systems Principles}, ser. SOSP '17.\hskip 1em plus 0.5em minus
  0.4em\relax New York, NY, USA: Association for Computing Machinery, 2017, p.
  441–459. [Online]. Available: \url{https://doi.org/10.1145/3132747.3132769}
\BIBentrySTDinterwordspacing

\bibitem{bu2022scalable}
\BIBentryALTinterwordspacing
Z.~Bu, J.~Mao, and S.~Xu, ``Scalable and efficient training of large
  convolutional neural networks with differential privacy,'' in \emph{Advances
  in Neural Information Processing Systems 35: Annual Conference on Neural
  Information Processing Systems 2022, NeurIPS 2022, New Orleans, LA, USA,
  November 28 - December 9, 2022}, 2022. [Online]. Available:
  \url{http://papers.nips.cc/paper\_files/paper/2022/hash/fa5617c176e76fee83f3f9947fdf9f3f-Abstract-Conference.html}
\BIBentrySTDinterwordspacing

\bibitem{bu2023differentially}
\BIBentryALTinterwordspacing
Z.~Bu, Y.~Wang, S.~Zha, and G.~Karypis, ``Differentially private optimization
  on large model at small cost,'' in \emph{International Conference on Machine
  Learning, {ICML} 2023, 23-29 July 2023, Honolulu, Hawaii, {USA}}, ser.
  Proceedings of Machine Learning Research, A.~Krause, E.~Brunskill, K.~Cho,
  B.~Engelhardt, S.~Sabato, and J.~Scarlett, Eds., vol. 202.\hskip 1em plus
  0.5em minus 0.4em\relax {PMLR}, 2023, pp. 3192--3218. [Online]. Available:
  \url{https://proceedings.mlr.press/v202/bu23a.html}
\BIBentrySTDinterwordspacing

\bibitem{Carlini2019secret}
N.~Carlini, C.~Liu, {\'{U}}.~Erlingsson, J.~Kos, and D.~Song, ``The secret
  sharer: Evaluating and testing unintended memorization in neural networks,''
  in \emph{28th {USENIX} Security Symposium, {USENIX} Security 2019, Santa
  Clara, CA, USA, August 14-16, 2019}, N.~Heninger and P.~Traynor, Eds.\hskip
  1em plus 0.5em minus 0.4em\relax {USENIX} Association, 2019, pp. 267--284.

\bibitem{cheu2019distributed}
A.~Cheu, A.~Smith, J.~Ullman, D.~Zeber, and M.~Zhilyaev, ``Distributed
  differential privacy via shuffling,'' in \emph{Advances in Cryptology --
  EUROCRYPT 2019}, Y.~Ishai and V.~Rijmen, Eds.\hskip 1em plus 0.5em minus
  0.4em\relax Cham: Springer International Publishing, 2019, pp. 375--403.

\bibitem{de2022unlocking}
S.~De, L.~Berrada, J.~Hayes, S.~L. Smith, and B.~Balle, ``Unlocking
  high-accuracy differentially private image classification through scale,''
  \emph{arXiv preprint arXiv:2204.13650}, 2022.

\bibitem{Devlin2019bert}
\BIBentryALTinterwordspacing
J.~Devlin, M.~Chang, K.~Lee, and K.~Toutanova, ``{BERT:} pre-training of deep
  bidirectional transformers for language understanding,'' in \emph{Proceedings
  of the 2019 Conference of the North American Chapter of the Association for
  Computational Linguistics: Human Language Technologies, {NAACL-HLT} 2019,
  Minneapolis, MN, USA, June 2-7, 2019, Volume 1 (Long and Short Papers)},
  J.~Burstein, C.~Doran, and T.~Solorio, Eds.\hskip 1em plus 0.5em minus
  0.4em\relax Association for Computational Linguistics, 2019, pp. 4171--4186.
  [Online]. Available: \url{https://doi.org/10.18653/v1/n19-1423}
\BIBentrySTDinterwordspacing

\bibitem{dong2021gaussian}
J.~Dong, A.~Roth, and W.~Su, ``Gaussian differential privacy,'' \emph{Journal
  of the Royal Statistical Society}, 2021.

\bibitem{Dosovitskiy2021vit}
\BIBentryALTinterwordspacing
A.~Dosovitskiy, L.~Beyer, A.~Kolesnikov, D.~Weissenborn, X.~Zhai,
  T.~Unterthiner, M.~Dehghani, M.~Minderer, G.~Heigold, S.~Gelly, J.~Uszkoreit,
  and N.~Houlsby, ``An image is worth 16x16 words: Transformers for image
  recognition at scale,'' in \emph{9th International Conference on Learning
  Representations, {ICLR} 2021, Virtual Event, Austria, May 3-7, 2021}.\hskip
  1em plus 0.5em minus 0.4em\relax OpenReview.net, 2021. [Online]. Available:
  \url{https://openreview.net/forum?id=YicbFdNTTy}
\BIBentrySTDinterwordspacing

\bibitem{dwork2006calibrating}
C.~Dwork, F.~McSherry, K.~Nissim, and A.~Smith, ``Calibrating noise to
  sensitivity in private data analysis,'' in \emph{Theory of cryptography
  conference}.\hskip 1em plus 0.5em minus 0.4em\relax Springer, 2006, pp.
  265--284.

\bibitem{Erlingsson2019Amplification}
U.~Erlingsson, V.~Feldman, I.~Mironov, A.~Raghunathan, K.~Talwar, and
  A.~Thakurta, ``Amplification by shuffling: from local to central differential
  privacy via anonymity,'' in \emph{Proceedings of the Thirtieth Annual
  ACM-SIAM Symposium on Discrete Algorithms}, ser. SODA '19.\hskip 1em plus
  0.5em minus 0.4em\relax USA: Society for Industrial and Applied Mathematics,
  2019, p. 2468–2479.

\bibitem{feldman2021hiding}
V.~Feldman, A.~McMillan, and K.~Talwar, ``Hiding among the clones: A simple and
  nearly optimal analysis of privacy amplification by shuffling,'' in
  \emph{2021 IEEE 62nd Annual Symposium on Foundations of Computer Science
  (FOCS)}, 2022, pp. 954--964.

\bibitem{feldman2022hiding}
------, ``{Hiding among the Clones: A Simple and Nearly Optimal Analysis of
  Privacy Amplification by Shuffling},'' in \emph{2021 IEEE 62nd Annual
  Symposium on Foundations of Computer Science (FOCS)}.\hskip 1em plus 0.5em
  minus 0.4em\relax IEEE, 2022, pp. 954--964.

\bibitem{Fenton1960}
L.~Fenton, ``The sum of log-normal probability distributions in scatter
  transmission systems,'' \emph{IRE Transactions on Communications Systems},
  vol.~8, no.~1, pp. 57--67, 1960.

\bibitem{Golatkar2022MixedDP}
A.~Golatkar, A.~Achille, Y.-X. Wang, A.~Roth, M.~Kearns, and S.~Soatto, ``Mixed
  differential privacy in computer vision,'' \emph{2022 IEEE/CVF Conference on
  Computer Vision and Pattern Recognition (CVPR)}, pp. 8366--8376, 2022.

\bibitem{goo2020lowrank}
M.~Gooneratne, K.~C. Sim, P.~Zadrazil, A.~Kabel, F.~Beaufays, and G.~Motta,
  ``Low-rank gradient approximation for memory-efficient on-device training of
  deep neural network,'' \emph{ICASSP 2020 - 2020 IEEE International Conference
  on Acoustics, Speech and Signal Processing (ICASSP)}, pp. 3017--3021, 2020.

\bibitem{PRVAccountant}
S.~Gopi, Y.~T. Lee, and L.~Wutschitz, ``Numerical composition of differential
  privacy,'' \emph{Advances in Neural Information Processing Systems}, vol.~34,
  pp. 11\,631--11\,642, 2021.

\bibitem{he2023exploring}
\BIBentryALTinterwordspacing
J.~He, X.~Li, D.~Yu, H.~Zhang, J.~Kulkarni, Y.~T. Lee, A.~Backurs, N.~Yu, and
  J.~Bian, ``Exploring the limits of differentially private deep learning with
  group-wise clipping,'' in \emph{The Eleventh International Conference on
  Learning Representations}, 2023. [Online]. Available:
  \url{https://openreview.net/forum?id=oze0clVGPeX}
\BIBentrySTDinterwordspacing

\bibitem{hoory-etal-2021-learning}
\BIBentryALTinterwordspacing
S.~Hoory, A.~Feder, A.~Tendler, S.~Erell, A.~Peled-Cohen, I.~Laish, H.~Nakhost,
  U.~Stemmer, A.~Benjamini, A.~Hassidim, and Y.~Matias, ``Learning and
  evaluating a differentially private pre-trained language model,'' in
  \emph{Findings of the Association for Computational Linguistics: EMNLP
  2021}.\hskip 1em plus 0.5em minus 0.4em\relax Punta Cana, Dominican Republic:
  Association for Computational Linguistics, Nov. 2021, pp. 1178--1189.
  [Online]. Available: \url{https://aclanthology.org/2021.findings-emnlp.102}
\BIBentrySTDinterwordspacing

\bibitem{Kairouz2021Practical}
\BIBentryALTinterwordspacing
P.~Kairouz, B.~Mcmahan, S.~Song, O.~Thakkar, A.~Thakurta, and Z.~Xu,
  ``Practical and private (deep) learning without sampling or shuffling,'' in
  \emph{Proceedings of the 38th International Conference on Machine Learning},
  ser. Proceedings of Machine Learning Research, M.~Meila and T.~Zhang, Eds.,
  vol. 139.\hskip 1em plus 0.5em minus 0.4em\relax PMLR, 18--24 Jul 2021, pp.
  5213--5225. [Online]. Available:
  \url{https://proceedings.mlr.press/v139/kairouz21b.html}
\BIBentrySTDinterwordspacing

\bibitem{pld3}
A.~Koskela and A.~Honkela, ``Computing differential privacy guarantees for
  heterogeneous compositions using fft,'' 2021.

\bibitem{pld1}
\BIBentryALTinterwordspacing
A.~Koskela, J.~J\"alk\"o, and A.~Honkela, ``Computing tight differential
  privacy guarantees using fft,'' pp. 2560--2569, 26--28 Aug 2020. [Online].
  Available: \url{https://proceedings.mlr.press/v108/koskela20b.html}
\BIBentrySTDinterwordspacing

\bibitem{pld2}
\BIBentryALTinterwordspacing
A.~Koskela, J.~J{\"a}lk{\"o}, L.~Prediger, and A.~Honkela, ``Tight differential
  privacy for discrete-valued mechanisms and for the subsampled gaussian
  mechanism using fft,'' pp. 3358--3366, 13--15 Apr 2021. [Online]. Available:
  \url{https://proceedings.mlr.press/v130/koskela21a.html}
\BIBentrySTDinterwordspacing

\bibitem{krizhevsky2009learning}
A.~Krizhevsky, G.~Hinton \emph{et~al.}, ``Learning multiple layers of features
  from tiny images,'' Citeseer, Tech. Rep., 2009.

\bibitem{Kurakin2022Toward}
\BIBentryALTinterwordspacing
A.~Kurakin, S.~Chien, S.~Song, R.~Geambasu, A.~Terzis, and A.~Thakurta,
  ``Toward training at imagenet scale with differential privacy,'' \emph{CoRR},
  vol. abs/2201.12328, 2022. [Online]. Available:
  \url{https://arxiv.org/abs/2201.12328}
\BIBentrySTDinterwordspacing

\bibitem{li2022low}
T.~Li, L.~Tan, Z.~Huang, Q.~Tao, Y.~Liu, and X.~Huang, ``Low dimensional
  trajectory hypothesis is true: Dnns can be trained in tiny subspaces,''
  \emph{IEEE Transactions on Pattern Analysis and Machine Intelligence}, 2022.

\bibitem{li2022large}
\BIBentryALTinterwordspacing
X.~Li, F.~Tramer, P.~Liang, and T.~Hashimoto, ``Large language models can be
  strong differentially private learners,'' in \emph{International Conference
  on Learning Representations}, 2022. [Online]. Available:
  \url{https://openreview.net/forum?id=bVuP3ltATMz}
\BIBentrySTDinterwordspacing

\bibitem{mironov2019renyi}
I.~Mironov, K.~Talwar, and L.~Zhang, ``R{\'{e}}nyi differential privacy of the
  sampled gaussian mechanism,'' \emph{CoRR}, vol. abs/1908.10530, 2019.

\bibitem{nan2021dart}
\BIBentryALTinterwordspacing
L.~Nan, D.~R. Radev, R.~Zhang, A.~Rau, A.~Sivaprasad, C.~Hsieh, X.~Tang,
  A.~Vyas, N.~Verma, P.~Krishna, Y.~Liu, N.~Irwanto, J.~Pan, F.~Rahman,
  A.~Zaidi, M.~Mutuma, Y.~Tarabar, A.~Gupta, T.~Yu, Y.~C. Tan, X.~V. Lin,
  C.~Xiong, R.~Socher, and N.~F. Rajani, ``{DART:} open-domain structured data
  record to text generation,'' in \emph{Proceedings of the 2021 Conference of
  the North American Chapter of the Association for Computational Linguistics:
  Human Language Technologies, {NAACL-HLT} 2021, Online, June 6-11, 2021},
  K.~Toutanova, A.~Rumshisky, L.~Zettlemoyer, D.~Hakkani{-}T{\"{u}}r,
  I.~Beltagy, S.~Bethard, R.~Cotterell, T.~Chakraborty, and Y.~Zhou, Eds.\hskip
  1em plus 0.5em minus 0.4em\relax Association for Computational Linguistics,
  2021, pp. 432--447. [Online]. Available:
  \url{https://doi.org/10.18653/v1/2021.naacl-main.37}
\BIBentrySTDinterwordspacing

\bibitem{nasr2023tight}
\BIBentryALTinterwordspacing
M.~Nasr, J.~Hayes, T.~Steinke, B.~Balle, F.~Tram{\`{e}}r, M.~Jagielski,
  N.~Carlini, and A.~Terzis, ``Tight auditing of differentially private machine
  learning,'' in \emph{32nd {USENIX} Security Symposium, {USENIX} Security
  2023, Anaheim, CA, USA, August 9-11, 2023}, J.~A. Calandrino and C.~Troncoso,
  Eds.\hskip 1em plus 0.5em minus 0.4em\relax {USENIX} Association, 2023, pp.
  1631--1648. [Online]. Available:
  \url{https://www.usenix.org/conference/usenixsecurity23/presentation/nasr}
\BIBentrySTDinterwordspacing

\bibitem{Novikova2017e2e}
\BIBentryALTinterwordspacing
J.~Novikova, O.~Dusek, and V.~Rieser, ``The {E2E} dataset: New challenges for
  end-to-end generation,'' in \emph{Proceedings of the 18th Annual SIGdial
  Meeting on Discourse and Dialogue, Saarbr{\"{u}}cken, Germany, August 15-17,
  2017}, K.~Jokinen, M.~Stede, D.~DeVault, and A.~Louis, Eds.\hskip 1em plus
  0.5em minus 0.4em\relax Association for Computational Linguistics, 2017, pp.
  201--206. [Online]. Available: \url{https://doi.org/10.18653/v1/w17-5525}
\BIBentrySTDinterwordspacing

\bibitem{radford2019language}
A.~Radford, J.~Wu, R.~Child, D.~Luan, D.~Amodei, and I.~Sutskever, ``Language
  models are unsupervised multitask learners,'' \emph{OpenAI Blog}, vol.
  1(8):9, 2019.

\bibitem{RahmanRLM18}
\BIBentryALTinterwordspacing
M.~A. Rahman, T.~Rahman, R.~Lagani{\`{e}}re, and N.~Mohammed, ``Membership
  inference attack against differentially private deep learning model,''
  \emph{Transactions on Data Privacy}, vol.~11, no.~1, pp. 61--79, 2018.
  [Online]. Available: \url{http://www.tdp.cat/issues16/tdp.a289a17.pdf}
\BIBentrySTDinterwordspacing

\bibitem{Sablayrolles2019WhiteboxVB}
A.~Sablayrolles, M.~Douze, C.~Schmid, Y.~Ollivier, and H.~J{\'e}gou,
  ``White-box vs black-box: Bayes optimal strategies for membership
  inference,'' in \emph{International Conference on Machine Learning}, 2019.

\bibitem{song2013stochastic}
S.~Song, K.~Chaudhuri, and A.~D. Sarwate, ``Stochastic gradient descent with
  differentially private updates,'' in \emph{2013 IEEE Global Conference on
  Signal and Information Processing}.\hskip 1em plus 0.5em minus 0.4em\relax
  IEEE, 2013, pp. 245--248.

\bibitem{Steinke2022CompositionOD}
\BIBentryALTinterwordspacing
T.~Steinke, ``Composition of differential privacy \& privacy amplification by
  subsampling,'' \emph{ArXiv}, vol. abs/2210.00597, 2022. [Online]. Available:
  \url{https://api.semanticscholar.org/CorpusID:252683292}
\BIBentrySTDinterwordspacing

\bibitem{vogels2019powersgd}
.~Thijs~Vogels, .~Sai Praneeth~Karimireddy, and M.~Jaggi, ``Powersgd: Practical
  low-rank gradient compression for distributed optimization,'' in
  \emph{Proceedings of the Advances in Neural Information Processing Systems
  (NeurIPS)}, 2019.

\bibitem{Vaswani2017AttentionIA}
A.~Vaswani, N.~M. Shazeer, N.~Parmar, J.~Uszkoreit, L.~Jones, A.~N. Gomez,
  L.~Kaiser, and I.~Polosukhin, ``Attention is all you need,'' in \emph{Neural
  Information Processing Systems}, 2017.

\bibitem{wang2019glue}
\BIBentryALTinterwordspacing
A.~Wang, A.~Singh, J.~Michael, F.~Hill, O.~Levy, and S.~R. Bowman, ``{GLUE:}
  {A} multi-task benchmark and analysis platform for natural language
  understanding,'' in \emph{7th International Conference on Learning
  Representations, {ICLR} 2019, New Orleans, LA, USA, May 6-9, 2019}.\hskip 1em
  plus 0.5em minus 0.4em\relax OpenReview.net, 2019. [Online]. Available:
  \url{https://openreview.net/forum?id=rJ4km2R5t7}
\BIBentrySTDinterwordspacing

\bibitem{wang2024unified}
C.~Wang, B.~Su, J.~Ye, R.~Shokri, and W.~Su, ``{Unified Enhancement of Privacy
  Bounds for Mixture Mechanisms via $ f $-Differential Privacy},''
  \emph{Advances in Neural Information Processing Systems (NeurIPS)}, vol.~36,
  2024.

\bibitem{wolf2020transformers}
\BIBentryALTinterwordspacing
T.~Wolf, L.~Debut, V.~Sanh, J.~Chaumond, C.~Delangue, A.~Moi, P.~Cistac,
  T.~Rault, R.~Louf, M.~Funtowicz, J.~Davison, S.~Shleifer, P.~von Platen,
  C.~Ma, Y.~Jernite, J.~Plu, C.~Xu, T.~L. Scao, S.~Gugger, M.~Drame, Q.~Lhoest,
  and A.~M. Rush, ``Transformers: State-of-the-art natural language
  processing,'' in \emph{Proceedings of the 2020 Conference on Empirical
  Methods in Natural Language Processing: System Demonstrations}.\hskip 1em
  plus 0.5em minus 0.4em\relax Online: Association for Computational
  Linguistics, Oct. 2020, pp. 38--45. [Online]. Available:
  \url{https://www.aclweb.org/anthology/2020.emnlp-demos.6}
\BIBentrySTDinterwordspacing

\bibitem{xu2024Permutation}
H.~Xu, L.~Xiang, H.~Ye, D.~Yao, P.~Chu, and B.~Li, ``Permutation equivariance
  of transformers and its applications,'' in \emph{Proceedings of the IEEE /
  CVF Computer Vision and Pattern Recognition Conference}, 2024.

\bibitem{yousefpour2021opacus}
\BIBentryALTinterwordspacing
A.~Yousefpour, I.~Shilov, A.~Sablayrolles, D.~Testuggine, K.~Prasad, M.~Malek,
  J.~Nguyen, S.~Ghosh, A.~Bharadwaj, J.~Zhao, G.~Cormode, and I.~Mironov,
  ``Opacus: User-friendly differential privacy library in pytorch,'' in
  \emph{NeurIPS 2021 Workshop Privacy in Machine Learning}, 2021. [Online].
  Available: \url{https://openreview.net/forum?id=EopKEYBoI-}
\BIBentrySTDinterwordspacing

\bibitem{yu2022differentially}
\BIBentryALTinterwordspacing
D.~Yu, S.~Naik, A.~Backurs, S.~Gopi, H.~A. Inan, G.~Kamath, J.~Kulkarni, Y.~T.
  Lee, A.~Manoel, L.~Wutschitz, S.~Yekhanin, and H.~Zhang, ``Differentially
  private fine-tuning of language models,'' in \emph{International Conference
  on Learning Representations}, 2022. [Online]. Available:
  \url{https://openreview.net/forum?id=Q42f0dfjECO}
\BIBentrySTDinterwordspacing

\bibitem{yu2021not}
D.~Yu, H.~Zhang, W.~Chen, and T.-Y. Liu, ``Do not let privacy overbill utility:
  Gradient embedding perturbation for private learning,'' in
  \emph{International Conference on Learning Representations}, 2021.

\bibitem{yu2021large}
D.~Yu, H.~Zhang, W.~Chen, J.~Yin, and T.-Y. Liu, ``Large scale private learning
  via low-rank reparametrization,'' in \emph{International Conference on
  Machine Learning}.\hskip 1em plus 0.5em minus 0.4em\relax PMLR, 2021, pp.
  12\,208--12\,218.

\bibitem{Huanyu2021Wide}
\BIBentryALTinterwordspacing
H.~Zhang, I.~Mironov, and M.~Hejazinia, ``Wide network learning with
  differential privacy,'' \emph{CoRR}, vol. abs/2103.01294, 2021. [Online].
  Available: \url{https://arxiv.org/abs/2103.01294}
\BIBentrySTDinterwordspacing

\bibitem{lognormal}
\BIBentryALTinterwordspacing
X.~Zhang, ``A survey of the sum of lognormal distribution and some recent
  results,'' Ph.D. dissertation, University of British Columbia, 2022.
  [Online]. Available:
  \url{https://open.library.ubc.ca/collections/ubctheses/24/items/1.0421709}
\BIBentrySTDinterwordspacing

\bibitem{Zhang_2022}
\BIBentryALTinterwordspacing
------, ``A survey of the sum of lognormal distribution and some recent
  results,'' Ph.D. dissertation, University of British Columbia, 2022.
  [Online]. Available:
  \url{https://open.library.ubc.ca/collections/ubctheses/24/items/1.0421709}
\BIBentrySTDinterwordspacing

\bibitem{zhou2021bypassing}
Y.~Zhou, Z.~S. Wu, and A.~Banerjee, ``Bypassing the ambient dimension: Private
  sgd with gradient subspace identification,'' in \emph{International
  Conference on Learning Representations}, 2021.

\bibitem{Zhu2019DeepLF}
L.~Zhu, Z.~Liu, and S.~Han, ``Deep leakage from gradients,'' in \emph{Proc. of
  the Advances in Neural Information Processing Systems (NIPS)}, 2019.

\end{thebibliography}
